\documentclass[sigconf,nonacm]{acmart}

\usepackage{amsmath,amsfonts,bm}

\def\eqref#1{equation~\ref{#1}}

\def\1{\bm{1}}

\DeclareMathAlphabet{\mathsfit}{\encodingdefault}{\sfdefault}{m}{sl}
\SetMathAlphabet{\mathsfit}{bold}{\encodingdefault}{\sfdefault}{bx}{n}

\newcommand{\E}{\mathbb{E}}

\newcommand{\KL}{D_{\mathrm{KL}}}

\DeclareMathOperator{\Tr}{Tr}

\usepackage{amsmath,amsfonts}

\usepackage{amsthm}

\DeclareMathOperator{\Exp}{Exp}

\newtheorem{theorem}{\textbf{Theorem}}

\newtheorem{proposition}{\textbf{Proposition}}
\newtheorem{corollary}{\textbf{Corollary}}

\newtheorem{remark}{Remark}

\usepackage{float}

\usepackage{color}
\usepackage{tikz}
\usetikzlibrary{trees}
\usepackage{amsthm}
\usepackage{makecell}
\usepackage{tikz}
\usetikzlibrary{trees}
\usepackage{tabularx,colortbl}
\usepackage{threeparttable}
\usepackage{booktabs}
\usepackage{multirow}
\usepackage{subfig}
\usepackage{graphicx}
\usepackage{url}
\usepackage{enumitem}
\usepackage{tcolorbox}
\usepackage{xcolor}
\usepackage[ruled, vlined, linesnumbered]{algorithm2e}
\usepackage{hyperref}
\usepackage[noabbrev]{cleveref}
\crefname{equation}{Eq.}{Eqs.}
\usepackage{soul}
\usepackage{adjustbox}

\definecolor{DeepPink}{HTML}{FF1493}
\definecolor{Orchid}{HTML}{DA70D6}
\definecolor{Magenta}{HTML}{FF00FF}
\definecolor{Fuchsia}{HTML}{FF00FF}
\definecolor{LavenderPink}{HTML}{FFB6C1}
\definecolor{verylightgray}{rgb}{0.9, 0.9, 0.9}
\definecolor{lightred}{rgb}{1,0.8,0.8}

\AtBeginDocument{%
  }

\setcopyright{acmlicensed}
\copyrightyear{2018}
\acmYear{2018}
\acmDOI{XXXXXXX.XXXXXXX}
\acmConference[Conference acronym 'XX]{Make sure to enter the correct
  conference title from your rights confirmation email}{June 03--05,
  2018}{Woodstock, NY}
\acmISBN{978-1-4503-XXXX-X/2018/06}

\begin{document}

\title{GeoIB: Geometry-Aware Information Bottleneck via Statistical-Manifold Compression}

	\author{Weiqi Wang}
\affiliation{%
	\institution{University of Technology Sydney}
	\city{Sydney}
	\state{NSW}
	\country{Australia}}

\author{Zhiyi Tian}
\affiliation{%
	\institution{University of Technology Sydney}
	\city{Sydney}
	\state{NSW}
	\country{Australia}}

\author{Chenhan Zhang}
\affiliation{%
	\institution{University of Technology Sydney}
	\city{Sydney}
	\state{NSW}
	\country{Australia}}

\author{Shui Yu}
\affiliation{%
	\institution{University of Technology Sydney}
	\city{Sydney}
	\state{NSW}
	\country{Australia}}



\begin{abstract}
 
Information Bottleneck (IB) is widely used, but in deep learning, it is usually implemented through tractable surrogates, such as variational bounds or neural mutual information (MI) estimators, rather than directly controlling the MI $I(X;Z)$ itself. The looseness and estimator-dependent bias can make IB ``compression'' only indirectly controlled and optimization fragile. 

We revisit the IB problem through the lens of information geometry and propose a \textbf{Geo}metric \textbf{I}nformation \textbf{B}ottleneck (\textbf{GeoIB}) that dispenses with mutual information (MI) estimation. We show that $I(X;Z)$ and $I(Z;Y)$ admit exact projection forms as minimal Kullback–Leibler (KL) distances from the joint distributions to their respective independence manifolds. Guided by this view, GeoIB controls information compression with two complementary terms: (i) a distribution-level Fisher–Rao (FR) discrepancy, which matches KL to second order and is reparameterization-invariant; and (ii) a geometry-level Jacobian–Frobenius (JF) term that provides a local capacity-type upper bound on $I_\phi(Z;X)$ by penalizing pullback volume expansion of the encoder. We further derive a natural-gradient optimizer consistent with the FR metric and prove that the standard additive natural-gradient step is first-order equivalent to the geodesic update. We conducted extensive experiments and observed that the GeoIB achieves a better trade-off between prediction accuracy and compression ratio in the information plane than the mainstream IB baselines on popular datasets. GeoIB improves invariance and optimization stability by unifying distributional and geometric regularization under a single bottleneck multiplier. The source code of GeoIB is released at \url{https://anonymous.4open.science/r/G-IB-0569}.




\end{abstract}




\maketitle
\thispagestyle{plain}
\pagestyle{plain} 
 \section{Introduction}


The Information Bottleneck (IB) principle \citep{tishby2000information} casts representation learning as extracting a representation $Z$ from $X$ that preserves only what is useful for predicting $Y$. Concretely, one seeks an encoder $q_\phi (z\mid x)$ such that $Z$ carries as much information about $Y$ as possible while remaining maximally compressed with respect to $X$, which can be formulated as:
\begin{equation}
	\max_{\phi }\ I_\phi (Z;Y)\quad \text{s.t.}\quad I_{\phi}(X;Z)\le R,
\end{equation}
where $I(\cdot;\cdot)$ denotes the mutual information, $\phi$ are the parameters of the encoder, and $R$ sets the compression budget. Here, $I_\phi (\cdot;\cdot)$ is computed under the data distribution $p(x,y)$ and the encoder $q_{\phi}(z \mid x)$.
To address this constrained optimization problem, the IB method (\citep{tishby2000information,alemi2016deep}) introduces a positive Lagrange multiplier $\beta$, transforming the problem into 
\begin{equation}
	\min_{\phi} -I_{\phi}(Z;Y) + \beta I_{\phi}(X;Z),
\end{equation}
where $\beta \geq 0 $ (the ``bottleneck'' parameter) balances predictive sufficiency against compression. 


The IB principle is appealing because it formalizes what constitutes a useful representation via a fundamental balance between compression and predictive sufficiency \citep{alemi2016deep,tishby2000information,wu2020learnability}. A wide range of variants of IB \citep{wan2021multi, yang2025structured, yu2024cauchyschwarz, zhai2022adversarial} have been adopted across diverse applications, including image segmentation \citep{xu2024sctnet}, domain generalization \citep{li2022invariant}, semantic communication \citep{xie2023robust,wang2024scu}, and privacy compression \citep{dubois2021lossy,razeghi2023bottlenecks}. Moreover, prior work \citep{shwartz2017opening} suggests that IB provides a principled lens for interpreting certain training dynamics of deep neural networks and unveil universal attrition to interpret vision transformers \citep{hong2025comprehensive}.

\noindent
\textbf{Research Gap.}
Despite this success, most practical IB implementations optimize Euclidean surrogates of mutual information (MI), e.g., variational bounds as in VIB \citep{alemi2016deep} or neural MI estimators such as MINE \citep{belghazi2018mutual}. These surrogates disregard the statistical-manifold geometry of the encoder or posterior family, thus offering no explicit geometric guarantees over $I(X;Z)$, weakening the learning signal even at high MI, and decreasing both utility and compression. Recent extensions, such as structure IB \citep{yang2025structured, hu2024structured}, try to extract the structure information from the input. However, they still rely on Euclidean MI proxies, often degrading accuracy dramatically in strong-compression regime and making results highly sensitive to the compressive parameter $\beta$. This motivates the following research question.
 
\noindent
\textbf{Research Question.} \ul{How can we design an information bottleneck that remains stable, robust, and controllable under strong compression in high-dimensional representation learning?}

 \begin{figure*}[t]
 	\centering
 	\includegraphics[width=0.93\linewidth]{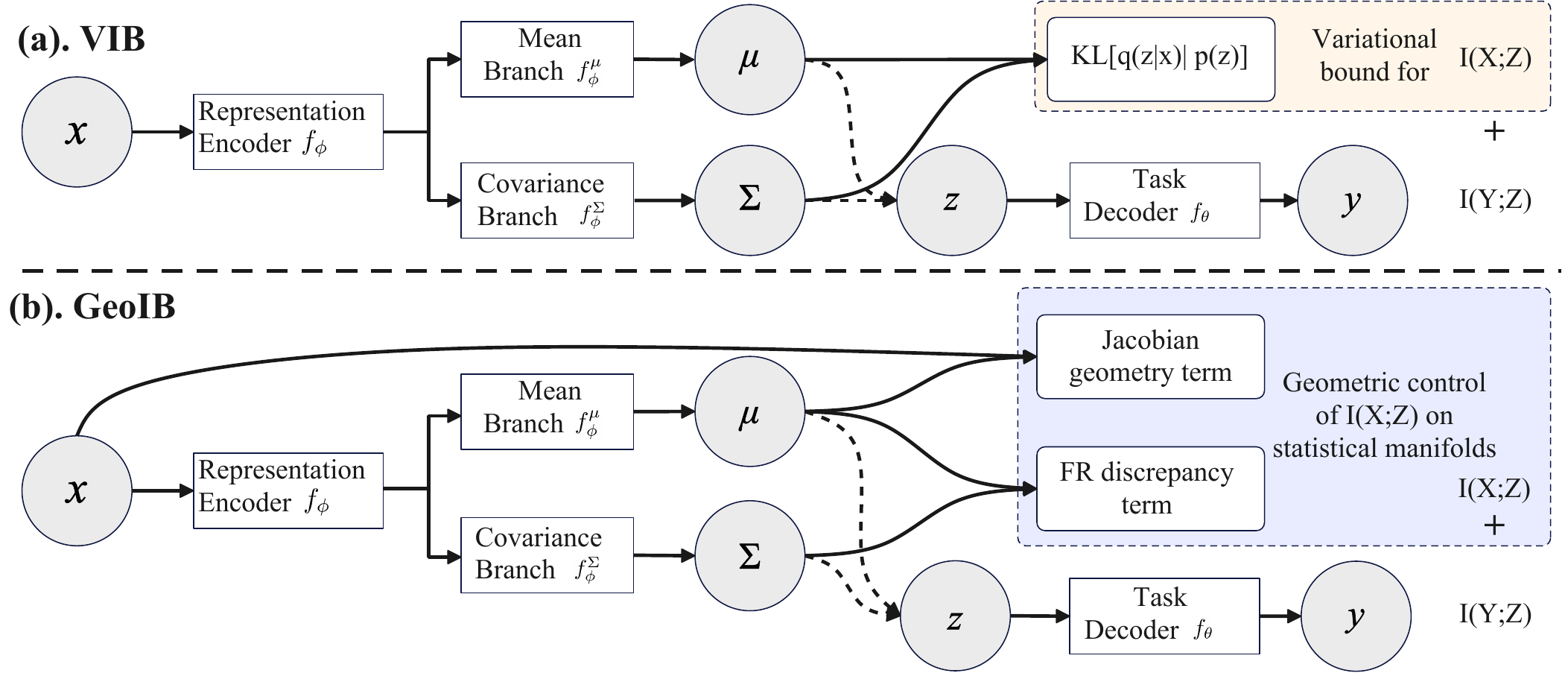}
 	\vspace{-4mm}
 	\caption{Comparison of VIB and GeoIB. Both models parametrize the encoder as $q_\phi(z\mid x)=\mathcal N(\mu(x),\mathrm{diag}(\sigma^2(x)))$ by a network $f_{\phi}$ and use a task decoder $p_{\theta}(y \mid z)$ to increase $I(Z;Y)$ by a network $f_{\theta}$. (a) VIB: compression is enforced by the variational upper bound. (b) GeoIB: replaces explicit MI estimation with two geometry-aware penalties computed deterministically on statistical manifolds: a Fisher–Rao quadratic proxy $\mathcal L_{\mathrm{FR}}$ and a Jacobian-Frobenius term $\mathcal L_{\mathrm{JF}}$. Solid arrows denote deterministic mappings; dashed arrows indicate reparameterized sampling $z=\mu+\sigma\odot\varepsilon$.}
 	\label{fig:figure1}
 \end{figure*}

\noindent
\textbf{Our Work.}
In this paper we introduce the Geometric Information Bottleneck (GeoIB), which reframes IB through the lens of statistical–manifold geometry. We first establish exact projection characterizations: both $I(X;Z)$ and $I(Z;Y)$ can be written as minimal Kullback–Leibler (KL) distances from the corresponding joint distributions to their independence submanifolds. Then, we design the GeoIB method, which regulates compression via two complementary components: (i) a \textit{distribution-level} Fisher–Rao (FR) discrepancy that agrees with KL to second order and is invariant under smooth reparameterizations of $z$; and (ii) a \textit{geometry-level} Jacobian–Frobenius (JF) penalty that yields a local capacity–type upper bound on $I_\phi(Z;X)$ by discouraging pullback volume expansion of the encoder. The two components make compression geometry consistent, stable to optimize, and faithful to the intended notion of information removal (distributional contraction plus capacity control). Finally, we derive an optimizer with respect to the Fisher–Rao (FR) metric whose update direction is the natural gradient, and we prove first-order equivalence with the geodesic update. 

To summarize, we make the following contributions:
\begin{itemize}[itemsep=0pt, parsep=0pt, leftmargin=*]
\item \textbf{Geometric Reformulation of IB.}
We show that both $I(Z;X)$ and $I(Z;Y)$ admit exact projection forms as minimal KL distances from the joint distributions to their respective independence manifolds, clarifying the geometric structure underlying the IB principle.

\item \textbf{A GeoIB Solution.}
We propose GeoIB, which controls compression via two complementary terms: (i) a distribution-level Fisher--Rao (FR) discrepancy and (ii) a geometry-level Jacobian--Frobenius (JF) penalty. 

\item \textbf{An Optimization Method for GeoIB.} We derive a natural-gradient optimizer consistent with the FR metric and prove that: the standard additive natural-gradient step is first-order equivalent to the geodesic (exponential-map) update.

\item \textbf{Empirical Validation.}
We conducted extensive experiments to compare with representative benchmarks. GeoIB attains favorable accuracy--compression trade-offs in the information plane relative to the state-of-the-art IB baselines, with improved robustness in strong compression regimes.
\end{itemize}

\section{Related Work} \label{related_work}

\subsection{Information Bottleneck}
The Information Bottleneck (IB) Lagrangian \citep{tishby2000information} has been widely studied in representation learning \citep{achille2018information,rosati2024representation} and practical training techniques \citep{xu2022infoat,li2025contrastive}. 
A practical deep implementation is the Variational Information Bottleneck (VIB) \citep{alemi2016deep}, which introduces a variational encoder $q_\phi(z\mid x)$, a prior $r(z)$, and a decoder/classifier $q_\theta(y\mid z)$, and uses the following bounds:
\begin{equation} \label{I_ZX}
	\begin{aligned}
		I(Z;X)
		&= \mathbb{E}_{p(x)}\!\big[ \mathrm{KL}\big(q_\phi(z\mid x)\,\|\,q_\phi(z)\big) \big] \\
		& \le \mathbb{E}_{p(x)}\!\big[ \mathrm{KL}\big(q_\phi(z\mid x)\,\|\,r(z)\big) \big] \\
		&= \mathbb{E}_{p(x)q_\phi(z\mid x)} \log q_\phi(z\mid x) \;-\; \mathbb{E}_{q_\phi(z)} \log r(z), 
	\end{aligned}
\end{equation}
\begin{equation} \label{I_ZY}
	\begin{aligned}
		I(Z;Y)
		&= H(Y) - H(Y\mid Z)
		\;\ge\; H(Y) + \mathbb{E}_{p(x,y)\,q_\phi(z\mid x)} \log q_\theta(y\mid z).  
	\end{aligned}
\end{equation}
Here $q_\phi(z)=\int q_\phi(z\mid x)p(x)\,dx$ is the encoder marginal. \Cref{I_ZX} uses $\mathrm{KL}\big(q_\phi(z)\,\|\,r(z)\big)\ge 0$ (with equality if $r(z)=q_\phi(z)$). \Cref{I_ZY} follows from the non-negativity of $\mathrm{KL}\!\big(p(y\mid z)\,\|\,q_\theta(y\mid z)\big)$ (cross-entropy bound).

Beyond VIB, a complementary line of work replaces the prior–KL surrogate with neural mutual-information estimators \citep{butakov2024mutual,kolchinsky2019nonlinear,franzese2023minde}. These methods train a critic to optimize variational bounds on MI and related quantities, enabling more direct optimization of the IB objective. Representative examples include MINE \citep{belghazi2018mutual}, which maximizes a Donsker--Varadhan lower bound; NWJ and related $f$-divergence bounds \citep{letizia2024mutual}; and CLUB, which provides an upper bound useful for penalizing $I(Z;X)$ \citep{cheng2020club}. These estimators remove the need to choose a prior $r(z)$, but introduce practical trade-offs (bias/variance, saturation for large MI, dependence on negatives and critic capacity), as analyzed by \citet{mcallester2020formal}.

While effective, most of the above approaches hinge on explicit or variational MI estimates in high dimensions, which can be brittle. We exploit the \emph{geometry} of the statistical manifold to implement information bottleneck. Specifically, in information geometry, the Fisher--Rao metric endows distributions with a Riemannian structure under which $\mathrm{KL}$ is locally the squared geodesic distance \citep{amari2000methods}; thus, a Fisher--Rao (FR) discrepancy $d_{\mathrm{FR}}^2\!\big(q_\phi(z\mid x),r(z)\big)$ provides a reparameterization-invariant surrogate for the $I(Z;X)$ compression term. Complementarily, viewing the encoder mean map $\mu_\phi: \mathcal X\!\to\!\mathcal Z$ as inducing a pullback metric $J_{\mu}^\top J_{\mu}$ suggests penalizing local volume distortion, which connects to contractive and Jacobian-based regularization \citep{ross2018improving}. We present the detailed introduction of the geometric information bottleneck in the methodology section.

\subsection{Information-Geometric Optimization}

Information-geometric optimization (IGO) views learning and optimization as moving on a statistical manifold of probability distributions equipped with the Fisher information as a Riemannian metric \citep{amari2000methods,amari2016information,amari2002information,ollivier2017information}. In this setting, the natural gradient \citep{amari1998natural} is the steepest-descent direction measured in the intrinsic geometry of the distribution family, rather than in an ambient Euclidean parameterization. Concretely, for a parametric family ${p_\theta}$ and an objective $\mathcal{L}(\theta)$, the natural-gradient update takes the form
\begin{equation}
\theta \leftarrow \theta - \eta\, F(\theta)^{-1}\nabla_\theta \mathcal{L}(\theta),
\end{equation}
where $F(\theta)$ is the Fisher information matrix. This update is invariant to smooth reparameterizations of $\theta$ and can be interpreted as following a locally KL-consistent direction, which often yields improved conditioning and stability compared to Euclidean gradients \citep{amari1998natural,amari2000methods}. 

 
\noindent
\textbf{Key Differences.}
Our method is inspired by the same information-geometric foundations, but it departs from IGO in both goal and use of geometry. First, IGO is primarily an optimization principle: it prescribes how to update distribution parameters so that progress is invariant to parameterization and respects the intrinsic geometry of the chosen distribution family \citep{ollivier2017information}. By contrast, our method is an information bottleneck method for representation learning: its central contribution is to redesign the compression mechanism so that it is geometry consistent. Second, the objects being optimized differ: IGO typically optimizes a search distribution over candidate solutions in a black-box setting \citep{ollivier2017information}, whereas we optimize an encoder posterior family $q_\phi(z\mid x)$ coupled with a task predictor, with the explicit aim of regulating the IB trade-off. Third, the invariances emphasized differ accordingly: IGO's defining invariances target objective-driven black-box optimization on distributions \citep{ollivier2017information}, while we leverage information geometry to obtain reparameterization-robust compression and capacity control for learned representations.



\section{Problem Statement from a Geometric View}\label{sec:problem}

Let $\mathcal{P}$ be the statistical manifold of all joint distributions over $(X,Z)$ and let the independence manifold be $\mathcal{I}_{XZ}=\{\,q(x)r(z)\,\}$. In exponential (e-) coordinates $\mathcal{I}_{XZ}$ is a e-flat submanifold \citep{amari2000methods,amari2016information}. For any $q$ and $r$, we have the information-geometric Pythagorean identity \citep{amari2000methods}
\begin{equation}  \label{eq:ig-pyth}
	\begin{aligned}
		\mathrm{KL}\!\big(p_\phi(x,z)\,\|\,q(x)r(z)\big)
		= & \underbrace{\mathrm{KL}\!\big(p_\phi(x,z)\,\|\,p(x)p_\phi(z)\big)}_{=\,I_\phi(X;Z)} \\
		&+ \underbrace{\mathrm{KL}\!\big(p(x)p_\phi(z)\,\|\,q(x)r(z)\big)}_{\ge 0},
	\end{aligned}
\end{equation}
whence $I_\phi(Z;X)=\min_{q,r}\mathrm{KL}\!\big(p_\phi(x,z)\,\|\,q(x)r(z)\big)$ with minimizer $(q,r)=(p(x),p_\phi(z))$. An identical relation holds for $(Y,Z)$ by replacing $x$ with $y$, $q$ with $q'$, and $r$ with $r'$:
\begin{equation} \label{eq_izy}
	\mathrm{KL}\!\big(p_\phi(y,z)\,\|\,q'(y)r'(z)\big)
	= I_\phi(Z;Y)+\mathrm{KL}\!\big(p(y)p_\phi(z)\,\|\,q'(y)r'(z)\big).
\end{equation}

\begin{figure}[t]
	\centering
	\includegraphics[width=0.91\linewidth]{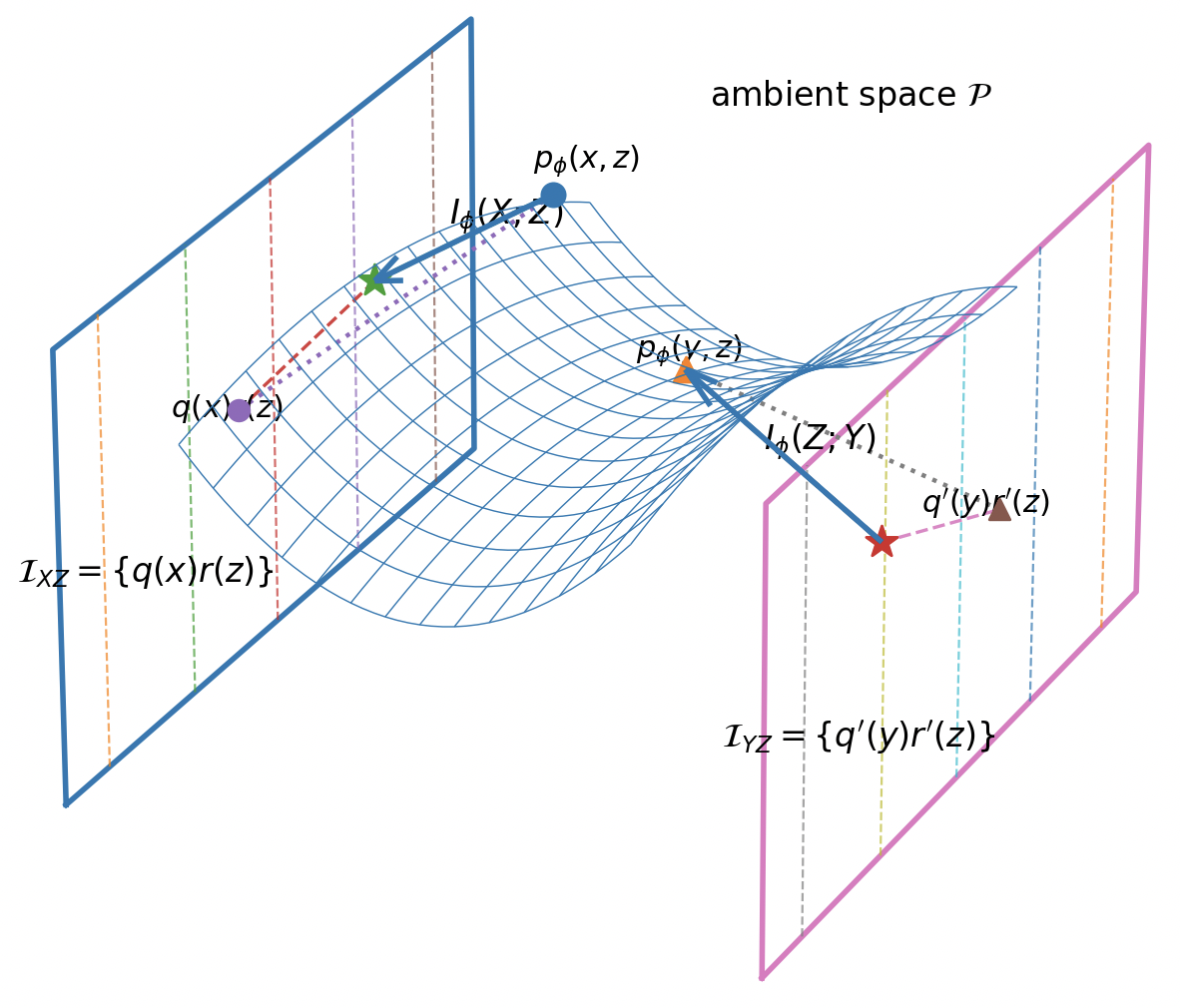}
	\vspace{-4mm}
	\caption{Information-geometric view of the Information Bottleneck objective. The ambient statistical manifold $\mathcal{P}$ contains the joint distributions $p_\phi(x,z)$ and $p_\phi(y,z)$. The blue and pink planes represent the e-flat independence manifolds $\mathcal{I}_{XZ}$ and $\mathcal{I}_{YZ}$. The arrows indicate KL I-projections onto these product manifolds, whose minimizers are $p(x)p_\phi(z)$ and $p(y)p_\phi(z)$; the corresponding KL distances equal $I_\phi(Z;X)$ and $I_\phi(Z;Y)$ via the information-geometric Pythagorean identity. Minimizing $\mathcal{L}_{\mathrm{IB}}(\phi)=\beta\,I_\phi(Z;X)-I_\phi(Z;Y)$ therefore pushes $p_\phi(x,z)$ towards $\mathcal{I}_{XZ}$ while pulling $p_\phi(y,z)$ away from $\mathcal{I}_{YZ}$.  
	}
	\label{fig:geoibproblem}
\end{figure}

We assume absolute continuity so that all KL terms are finite and Fubini's theorem \citep{kallenberg1997foundations} applies; in particular $p_\phi(x,z)\ll q(x)r(z)$ and $p_\phi(y,z)\ll q'(y)r'(z)$ for candidate product measures. We provide the detailed proof of \Cref{eq:ig-pyth} in \Cref{proof_of_ig_pyth}, and \Cref{eq_izy} can also be proved in the same way. The Information Bottleneck loss \citep{tishby2000information}, $\mathcal{L}_{\mathrm{IB}}(\phi)=\beta\,I_\phi(Z;X)-I_\phi(Z;Y)$, can thus be written exactly as
\begin{equation}\label{eq:ib-projection}
	\small
	\mathcal{L}_{\mathrm{IB}}(\phi)
	= \beta \min_{q,r}\ \mathrm{KL}\!\big(p_\phi(x,z)\,\|\,q(x)r(z)\big)
	\;-\; \min_{q',r'}\ \mathrm{KL}\!\big(p_\phi(y,z)\,\|\,q'(y)r'(z)\big),
\end{equation}
where the inner minima are achieved at $(q,r)=(p(x),p_\phi(z))$ and $(q',r')=(p(y),p_\phi(z))$. Optimizing \Cref{eq:ib-projection} over $\phi$ is therefore equivalent to:
\begin{itemize}
	\item \emph{push $p_\phi(x,z)$ toward the independence manifold} $\mathcal{I}_{XZ}$ by minimizing $\beta\,\mathrm{KL}\!\big(p_\phi(x,z)\,\|\,p(x)p_\phi(z)\big)$;
	\item \emph{pull $p_\phi(y,z)$ away from its independence manifold} $\mathcal{I}_{YZ}=\{q'(y)r'(z)\}$ by maximizing $\mathrm{KL}\!\big(p_\phi(y,z)\,\|\,p(y)p_\phi(z)\big)$.
\end{itemize}

We illustrate an example of geometric information bottleneck objectives in \Cref{fig:geoibproblem}. This geometric reformulation provides a concrete push–pull interpretation of the IB trade-off, and separates the geometry of the inner minimizers from the outer optimization over $\phi$ by expressing mutual information as a KL projection onto product manifolds, clarifying what is being optimized and where the optima occurs. Moreover, because the independence manifolds are e-flat, the formulation enables direct use of standard information-geometric tools such as KL projections and Pythagorean decompositions, which can simplify analysis and motivate principled approximations.



\section{Geometric Information Bottleneck Method} \label{OUL_approach}

In \Cref{sec:problem}, we recast the IB objective through information geometry: the mutual information $I_\phi(X;Z)$ and $I_\phi(Z;Y)$ are the minimal KL distances from the joint distributions $p_\phi(x,z)$ and $p_\phi(y,z)$ to their respective independence manifolds, turning the IB Lagrangian into a difference of projection distances. Based on this, our Geometric IB (GeoIB) controls compression with two complementary terms: (i) a \emph{distribution-level} Fisher--Rao (FR) discrepancy and (ii) a \emph{geometry-level} Jacobian–Frobenius (JF) term. During training, we propose a natural gradient descent method that combines the geometric information to achieve a better optimization effect.

\subsection{Distribution Proxy via the Fisher--Rao (FR) Quadratic}
We approximate the conditional–marginal divergence that defines the compression term by the local second-order FR metric:

\begin{equation}
\begin{aligned}
	I_\phi(Z;X)
	&= \E_{p(x)} \KL\!\big(q_\phi(z|x)\,\|\,p_\phi(z)\big)
	\approx \tfrac12\,\E_{p(x)} d_{\mathrm{FR} }\!\big(q_\phi(z|x), r(z)\big)^2,
	\label{eq:fr-proxy}
\end{aligned}
\end{equation}
where \(r(z)\) is a reference marginal (e.g., a standard normal or a learned prior), and $d_{\mathrm{FR}}(\cdot,\cdot)$ denotes the Fisher–Rao geodesic distance between distributions, i.e., the Riemannian distance induced by the Fisher information metric. The approximation follows from the local equivalence for smooth parametric families $\{p_\theta\}$, for $\theta'=\theta+\Delta$ with $\|\Delta\|$ small:
\begin{equation}
	\begin{aligned}
	\label{eq_smooth_p_theta_prime}
	\KL\!\big(p_{\theta'}\,\|\,p_\theta\big)
	= \tfrac12\,\Delta^\top F(\theta)\,\Delta + o(\|\Delta\|^2)
	= \tfrac12\, d_{\mathrm{FR}}\!\big(p_{\theta'},p_\theta\big)^2 + o(\|\Delta\|^2),
\end{aligned}
\end{equation}
with \(F(\theta)\) the Fisher information. Replacing \(p_{\theta'}\) by \(q_\phi(z|x)\) and \(p_\theta\) by \(r(z)\) in \Cref{eq_smooth_p_theta_prime} yields \Cref{eq:fr-proxy}. We provide the proof in \Cref{proof_of_FR}.

\noindent
\textbf{Diagonal Gaussian Example.}
If \(q_\phi(z|x)=\mathcal{N}\!\big(\mu_\phi(x),\mathrm{diag}(\sigma_\phi^2(x))\big)\) and \(r(z)=\mathcal{N}(0,I)\), the exact KL is
\begin{equation}
	\begin{aligned}
	\KL\!\big(q_\phi(z|x)\,\|\,\mathcal{N}(0,I)\big)
	=\tfrac12 \sum_{j=1}^{d_z} \Big(\mu_j(x)^2+\sigma_j(x)^2 - \log \sigma_j(x)^2 -1\Big).
	\label{eq:fr-gauss}
\end{aligned}
\end{equation}
Expanding \Cref{eq:fr-gauss} at \(\mu=0,\ \sigma^2=1\) gives
\begin{equation}
	\begin{aligned}
\KL\!\big(q_\phi\| \mathcal N(0,I)\big)
= \tfrac12\|\mu\|_2^2 \;+\; \tfrac14 \|\log \sigma^2\|_2^2 \;+\; O\!\big(\|\log \sigma^2\|_2^3\big),
\end{aligned}
\end{equation}
which equals \(\tfrac12\,d_{\mathrm{FR}}\!\big(q_\phi,\mathcal N(0,I)\big)^2\) up to second order.
In practice, when \(r=\mathcal N(0,I)\) we can optimize the closed-form KL in \Cref{eq:fr-gauss}:
\(\mathcal{L}_{\mathrm{KL}}(\phi)=\widehat{\E}_{x}\,\KL\!\big(q_\phi(z|x)\,\|\,r(z)\big)\).
If \(r\) is learned/complex (e.g., VampPrior/flow), we may optimize the FR proxy
\begin{equation}
	\begin{aligned}
\mathcal{L}_{\mathrm{FR}}(\phi)=\tfrac12\,\widehat{\E}_{x}\, d_{\mathrm{FR}}\!\big(q_\phi(z|x), r(z)\big)^2,
\end{aligned}
\end{equation}
and optionally monitor $\widehat I^{\mathrm{KL}}_{XZ}=\widehat{\E}_x\,\KL\!\big(q_\phi(z|x)\,\|\,\widehat p_\phi(z)\big)$ to gauge tightness.


\subsection{Geometric Bound via the Jacobian-Frobenius Term}
Assume a reparameterized encoder $ z \;=\; f_\phi(x) + \varepsilon,$ where $\varepsilon \sim \mathcal{N}\!\big(0,\Sigma(x)\big),$ and denote the Jacobian $J_f(x)=\partial f_\phi(x)/\partial x$.
The pullback metric on the input manifold is $g_x \;=\; J_f(x)^\top \Sigma(x)^{-1} J_f(x)$.

\noindent
\textbf{Local Capacity-type Upper Bound.}
Linearizing \(f_\phi\) around \(x\) and letting \(C_x\) be the local input covariance, a Gaussian channel upper bound yields
\begin{equation} 	\label{eq:jf-logdet-Cx}
	\begin{aligned}
	I_\phi(Z;X)
	& \leq \tfrac12\,\E_{p(x)}\!\left[\log\det\!\Big(I+\Sigma(x)^{-\tfrac12}J_f(x)\,C_x\,J_f(x)^\top\Sigma(x)^{-\tfrac12}\Big)\right].
\end{aligned}
\end{equation}
Under a unit local energy constraint \(C_x \preceq I\) (Loewner order) and the monotonicity of \(\log\det\) on PSD cone \(\mathbb{S}_+^d\), we obtain the pointwise bound
\begin{equation} 	\label{eq:jf-trace}
	\begin{aligned}
	I_\phi(Z;X)
	&\leq \tfrac12\,\E_{p(x)}\!\left[\log\det\!\Big(I+\Sigma(x)^{-\tfrac12}J_f(x)J_f(x)^\top\Sigma(x)^{-\tfrac12}\Big)\right]
	\\[-2mm]
	&= \tfrac12\,\E_{p(x)}\!\left[\log\det\!\Big(I+J_f(x)^\top\Sigma(x)^{-1}J_f(x)\Big)\right],  	 \\
	& \leq   \tfrac12\,\E_{p(x)}\,\Tr\!\big(\Sigma(x)^{-1}J_f(x)J_f(x)^\top\big)   \\
	&\;=\; \tfrac12\,\E_{p(x)}\big\|\Sigma(x)^{-\tfrac12}J_f(x)\big\|_F^2 \\
	&\;=:\; \tfrac12\,\mathrm{JF}(\phi).
\end{aligned}
\end{equation}

\textit{Proof sketch.}
Since $Z$ depends on $X$ only through $(f_\phi(X),\Sigma(X))$, we have the Markov chain
$X \to (f_\phi(X),\Sigma(X)) \to Z$ and thus $I(X;Z)=I\big((f_\phi(X),\Sigma(X));\,Z\big)$. Conditioning on $x$ and linearizing $f_\phi$ at $x$ while holding $\Sigma(x)$ fixed locally
yields a (local) linear Gaussian channel with gain $J_f(x)$ and noise covariance $\Sigma(x)$.
Under a unit local energy constraint on the input covariance $C_x \preceq I$, the Gaussian
maximizes entropy for fixed covariance, giving the log-det bound in \Cref{eq:jf-logdet-Cx}.
Using $\det(I+AB)=\det(I+BA)$, and applying
$\log\det(I+A)\le \Tr(A)$ for $A\succeq 0$ yields \Cref{eq:jf-trace}. \hfill$\square$

 
 \noindent
\textbf{Unbiased Hutchinson Estimator.}
The trace in \Cref{eq:jf-trace} can be estimated without forming explicit Jacobians.
For any $v\sim\mathcal{N}(0,I_{d_x})$, 
\begin{equation}
	\begin{aligned}
\mathbb{E}_v\big[\|\Sigma(x)^{-1/2}J_f(x)v\|_2^2\big]
= \mathrm{Tr}\!\big(\Sigma(x)^{-1}J_f(x)J_f(x)^\top\big).
\end{aligned}
\end{equation}
With $S$ i.i.d. probe vectors $\{v_s\}_{s=1}^S$, define the per-sample estimator
\begin{equation}
	\begin{aligned}
\widehat{\mathrm{JF}}(x)
:= \frac{1}{S}\sum_{s=1}^S \big\| \Sigma(x)^{-1/2}\, J_f(x)\, v_s \big\|_2^2,
\end{aligned}
\end{equation}
so that
\begin{equation}
	\begin{aligned}
	 \mathbb{E}_v[\widehat{\mathrm{JF}}(x)]
		= \mathrm{Tr}\!\big(\Sigma^{-1}J_f(x)J_f(x)^\top\big).
	\end{aligned}
\end{equation}
The training objective is the batch average
\begin{equation}
	\begin{aligned}
\mathcal{L}_{\mathrm{JF}}(\phi)
:= \widehat{\mathbb{E}}_{x\sim\text{batch}}\,\big[\widehat{\mathrm{JF}}(x)\big].
\end{aligned}
\end{equation}
In automatic differentiation frameworks, compute $J_f(x)v_s$ via Jacobian--vector
products (JVP), costing $O(S)$ forward-mode calls per $x$; $S=1$ or $2$ is typically sufficient.

 \noindent
\textbf{Isotropic/diagonal noise.}
If \(\Sigma(x)=\sigma(x)^2 I\), then
\begin{equation}
	\begin{aligned}
I_\phi(Z;X)
\;\le\;
\tfrac12\,\E_{p(x)}\log\det\!\Big(I+\tfrac{1}{\sigma(x)^2}J_fJ_f^\top\Big)
\;\le\;
\tfrac12\,\widehat{\E}_{x}\,\frac{\|J_f(x)\|_F^2}{\sigma(x)^2}.
\end{aligned}
\end{equation}
With a scalar floor $\sigma_{\min}^2 > 0$, this yields a simple, stable surrogate.

\textit{Geometric meaning.}
Because $g_x=J_f^\top\Sigma^{-1}J_f$, we have $\Tr(g_x)=\|\Sigma(x)^{-\tfrac12}J_f(x)\|_F^2$, the direction-averaged local stretch (Dirichlet energy density under the \(\Sigma^{-1}\) metric). Minimizing the JF term therefore controls average geodesic-length distortion, providing a principled compression surrogate for \(I_\phi(Z;X)\).



\subsection{Natural-Gradient Optimization for GeoIB}
\label{sec:natgrad}

Building on the above, we formulate the GeoIB objective as
\begin{equation}
	\label{eq:gib_objective0}
	\small
	\begin{aligned}
	\mathcal{L}_{\mathrm{GeoIB}}(\phi,\theta)
	\;=\; 
	\underbrace{ \E_{p(x,y)}\E_{q_\phi(z\mid x)}\big[-\log p_\theta(y\mid z)\big] }_{\uparrow\, I_\phi(Z;Y)} & \\
	\;+\;
	\underbrace{\beta\Big(\widehat{\mathcal{L}}_{\mathrm{FR}}(\phi)
		+\widehat{\mathcal{L}}_{\mathrm{JF}}(\phi)\Big)}_{\downarrow\, I_\phi(Z;X)}&,
	\end{aligned}
\end{equation}
where $\beta\!\ge\!0$ is the bottleneck multiplier. Minimizing $\E_{p(x,y)}\E_{q_\phi(z\mid x)}\big[-\log p_\theta(y\mid z)\big]$ reduces $H(Y\!\mid\!Z)$ and increases $I_\phi(Z;Y)$ (since $H(Y)$ is fixed); the FR and JF terms jointly penalize $I_\phi(Z;X)$.

\noindent
\textbf{Natural Gradient on the Encoder.}
Viewing $\{q_\phi(z\mid x)\}_\phi$ as a statistical manifold endowed with the Fisher--Rao metric, the natural gradient of any scalar objective $\mathcal{J}(\phi)$ is
\begin{equation}
	\begin{aligned}
	&\widetilde{\nabla}_\phi \mathcal{J}
	\;=\;
	F_\phi^{-1}\,\nabla_\phi \mathcal{J}, \\
	F_\phi
	\;:=\;
	\E_{p(x)}\E_{q_\phi(z\mid x)}
	\!\big[\, &
	\nabla_\phi \log q_\phi(z\mid x)\,\nabla_\phi \log q_\phi(z\mid x)^\top
	\,\big],
	\label{eq:natgrad-def}
\end{aligned}
\end{equation}
where the matrix $F_\phi$ in~\Cref{eq:natgrad-def} is the Fisher--Rao metric tensor. We update the encoder by preconditioning the Euclidean gradient of the full objective:
\begin{equation}
	\begin{aligned}
	\phi_{t+1}
	\;=\;
	\phi_t
	-
	\eta_\phi\,
	F_{\phi_t}^{-1}
	 \Big(
		\nabla_\phi\,\E_{x,y,z}[-\log p_\theta(y\mid z)] &\\
		+
		\beta\big[
		 \nabla_\phi \mathcal{L}_{\mathrm{FR}}
		+
		 \nabla_\phi \mathcal{L}_{\mathrm{JF}}
		\big]\Big)&,
	\label{eq:natgrad-step}
\end{aligned}
\end{equation}
which is first-order invariant under smooth reparameterizations of $\phi$ and couples the distribution-level FR and geometry-level JF signals through a single preconditioner $F_\phi$.
We can also prove that the additive natural-gradient step $\phi^{+}=\phi-\eta F_\phi^{-1}\nabla_\phi\mathcal{J}$ is a first-order approximation to the geodesic update in geometry. 

\begin{proposition}[Natural gradient equals the Riemannian gradient]\label{prop:riem-grad}
	Let $\mathcal M=\{p_\phi:\phi\in\Theta\subset\mathbb R^d\}$ be a regular statistical manifold
	endowed with the Fisher--Rao metric
	$g_\phi(u,v):=u^\top F_\phi v$, where
	$F_\phi=\E_{p(x)}\E_{q_\phi(z\mid x)}[\nabla_\phi\log q_\phi\,\nabla_\phi\log q_\phi^\top]$.
	For a scalar objective $\mathcal J:\mathcal M\to\mathbb R$, its Riemannian gradient at $\phi$ satisfies
	\begin{equation}
			\mathrm{grad}\,\mathcal J(\phi)\;=\;F_\phi^{-1}\,\nabla_\phi \mathcal J,
	\end{equation}
	i.e., the natural gradient $\widetilde\nabla_\phi\mathcal J:=F_\phi^{-1}\nabla_\phi\mathcal J$
	is exactly the Riemannian gradient on $(\mathcal M,g)$.
\end{proposition}
 
 
\textit{Proof sketch.} By definition, the Riemannian gradient $\mathrm{grad}\,\mathcal J(\phi)\in T_\phi\mathcal M\simeq\mathbb R^d$
 	is the unique vector field satisfying, for all tangent directions $v\in T_\phi\mathcal M$,
 	\[
 	\langle \mathrm{grad}\,\mathcal J(\phi),\, v\rangle_{g_\phi}
 	\;=\;
 	\mathrm D\mathcal J_\phi[v].
 	\]
 	Under the Fisher--Rao metric, the inner product is
 	$\langle u,v\rangle_{g_\phi}=u^\top F_\phi v$.
 	In coordinates, the differential equals the Euclidean pairing with the usual gradient:
 	$\mathrm d\mathcal J_\phi[v]=v^\top\nabla_\phi\mathcal J$.
 	Hence, for all $v$,
 	\[
 	v^\top F_\phi\,\mathrm{grad}\,\mathcal J(\phi)
 	\;=\;
 	v^\top \nabla_\phi \mathcal J.
 	\]
 	Since $F_\phi$ is positive definite at regular points, we conclude
 	$F_\phi\,\mathrm{grad}$ $\mathcal J(\phi)=\nabla_\phi\mathcal J$,
 	i.e. $\mathrm{grad}\,\mathcal J(\phi)=F_\phi^{-1}\nabla_\phi\mathcal J$.

\begin{proposition}[Steepest descent under the Fisher--Rao metric]\label{prop:steepest}
	Let $(\mathcal M,g)$ be endowed with the Fisher--Rao metric $g_\phi(u,v)=u^\top F_\phi v$ and let $\mathcal J$ be smooth.
	The direction of steepest descent per unit FR length solves
\begin{equation}
	\min_{\|v\|_{g_\phi}\le 1}\ \mathrm{D}\mathcal J(\phi)[v],
\end{equation}
and the (unit-norm) optimizer is $v_\star = - \frac{\mathrm{grad}\,\mathcal J(\phi)}{\|\mathrm{grad}\,\mathcal J(\phi)\|_{g_\phi}},$
	with the convention $v_\star=0$ if $\mathrm{grad}\,\mathcal J(\phi)=0$.
	In particular, by \Cref{prop:riem-grad}, its direction coincides with the negative natural gradient:
	$-\mathrm{grad}\,\mathcal J(\phi) \equiv -F_\phi^{-1}\nabla_\phi \mathcal J$.
\end{proposition}


\textit{Proof sketch.} By the Riemannian gradient definition in \citep{amari2000methods}, $\mathrm{D}\mathcal J(\phi)[v] \\=\langle \mathrm{grad}\,\mathcal J(\phi), v\rangle_{g_\phi}$. Cauchy--Schwarz gives $\langle \mathrm{grad}\,\mathcal J, v\rangle_{g_\phi} \ge \\-\|\mathrm{grad}\,\mathcal J\|_{g_\phi}\,\|v\|_{g_\phi} \ge -\|\mathrm{grad}\,\mathcal J\|_{g_\phi}$, with equality iff $v$ is collinear with $-\mathrm{grad}\,\mathcal J$ and $\|v\|_{g_\phi}=1$.
\hfill$\square$

By \Cref{prop:steepest}, the steepest descent direction per unit FR length is $-\mathrm{grad}\,\mathcal J(\phi)$. We thus update along this direction using the exponential map.

\begin{theorem}[Geodesic update via the exponential map]\label{thm:exp-geodesic}
	Let $\Exp_\phi:T_\phi\mathcal{M}\to\mathcal{M}$ be the Riemannian exponential map of the FR
	metric. The discrete update
	\[
	\phi^{+}
	\;=\;
	\Exp_{\phi}\!\big(-\eta\,\mathrm{grad}\,\mathcal{J}(\phi)\big)
	\]
	lies on the unique FR geodesic $\gamma$ starting at $\phi$ with initial velocity
	$\dot\gamma(0)=-\eta\,\mathrm{grad}\,\mathcal{J}(\phi)$; i.e., $\phi^{+}=\gamma(1)$.
\end{theorem}
See proof in \Cref{proof_of_t1}.

\begin{corollary}[First-order equivalence to the additive update]\label{cor:first-order}
	Let $\mathcal{R}_\phi$ be any retraction on $\mathcal{M}$ satisfying
	$\mathcal{R}_\phi(0)=\phi$ and $\mathrm{D}\mathcal{R}_\phi(0)=\mathrm{Id}$ (the exponential
	map is a canonical retraction). In local coordinates,
	\[
	\Exp_\phi\!\big(-\eta\,F_\phi^{-1}\nabla_\phi\mathcal{J}\big)
	\;=\;
	\phi - \eta\,F_\phi^{-1}\nabla_\phi\mathcal{J}
	\;+\; O(\eta^2).
	\]
	By \Cref{prop:riem-grad}, $\mathrm{grad}\,\mathcal J=F_\phi^{-1}\nabla_\phi\mathcal J$. 
	Hence the common additive natural-gradient step
	$\phi^{+}=\phi-\eta F_\phi^{-1}\nabla_\phi\mathcal{J}$ is a first-order approximation to the geodesic update.
\end{corollary}

\textbf{Natural Gradient on the Decoder.}
For the decoder, we use the natural gradient on $\{p_\theta(y\mid z)\}_\theta$:
\begin{equation}
	\begin{aligned}
	\theta_{t+1}
	\;=\;
	\theta_t - \eta_\theta\,F_\theta^{-1}\,\nabla_\theta \E_{x,y,z}[-\log p_\theta(y\mid z)], 
	\label{eq:decoder-natgrad}
\end{aligned}
\end{equation}
where $ F_\theta = \E_{q_{\phi}(z)}\E_{p_\theta(y\mid z)} \!\big[\nabla_\theta\log p_\theta(y\mid z)\nabla_\theta\log p_\theta(y\mid z)^\top\big]$. 
In practice, we compute $F_\theta^{-1}g$ using scalable approximations and solvers, such as K-FAC \citep{martens2015optimizing,martens2018kronecker}.
As the page limitation, we present the whole GeoIB algorithm in \Cref{alg_of_gib}.

\begin{table*}[t]
	\caption{General Evaluation Results on image datasets, MNIST and CIFAR10, and CelebA. Results in bold are the best; those in italics are the second best. \vspace{-4mm}}
	\label{tab_total}
	\resizebox{0.99\linewidth}{!}{
		\setlength\tabcolsep{6.pt}
		\begin{tabular}{ccccc cccc cccc}
			\toprule
			\multirow{2}{*} {Methods} & \multicolumn{4}{c} {MNIST, $\beta=0.0001$, $K =128$}& \multicolumn{4}{c} {CIFAR10, $\beta=0.0001$, $K=128$} & \multicolumn{4}{c} {CelebA, $\beta=0.0001$, $K=128$} \\
			\cmidrule(r){2-5}   \cmidrule(r){6-9}  \cmidrule(r){10-13}
			& Accuracy  &$I(X;Z)$   & MSE	& RT  & Accuracy    &$I(X;Z)$   & MSE  & RT & Accuracy  &$I(X;Z)$   & MSE & RT	   \\
			\midrule 
			VIB \citep{alemi2016deep}   & 98.72\%   & 1.81       & 0.034   & \textbf{199.84}           & \textit{82.65\% } & 0.87 & \textit{0.887 }   & \textbf{277.27}    & 95.85\% &0.55		&0.054  & \textbf{264.18} \\
			SIB  \citep{yang2025structured}  & \textit{99.08\%}   & 1.86         & 0.037  & 399.31      & 75.81\% & \textbf{0.63} &0.051  & 524.49     & \textbf{97.25\%} & \textbf{0.38} & \textbf{0.072}	& 456.92  \\
			MINE  \citep{belghazi2018mutual} & 98.85\% & 1.76     & 0.032      & \textit{234.04}     & 82.64\% & 0.88 & 0.806    & \textit{362.07}    & 96.29\%	&0.51 &	0.048  & \textit{314.26}  \\
			AIB \citep{zhai2022adversarial} & 99.01\%	   & \textit{1.69}  &  \textit{0.043}   & 293.56    & 82.63\% &\textit{0.84} &0.812 & 509.56   & 96.05\%&	0.52 &0.053  & 454.80 \\
			GeoIB (Our)       	 & \textbf{99.28\%}	& \textbf{1.69}  &  \textbf{0.043}  & 442.32 & \textbf{85.54\%} &1.01 &\textbf{0.899 }   & 556.08   & \textit{97.01}\%& \textit{0.47} &	\textit{0.066}   & 464.91  \\
			\bottomrule[1pt]
	\end{tabular}}
\end{table*}

\section{Experiments} \label{experiment_Setting}
 
  In this section, we conduct experiments to answer the following research questions (RQ) about GeoIB:
 \begin{itemize}[itemsep=0pt, parsep=0pt, leftmargin=*]
 	\item \textbf{RQ1}: How does the proposed GeoIB perform on information compression and prediction accuracy, as compared with the state-of-the-art IB solutions? (See \Cref{overall_eva,info_plane})
 	\item \textbf{RQ2}: How do different hyperparameters, such as the Lagrange multiplier $\beta$ and representation dimensionality $K$, influence the GeoIB? (See \Cref{eva_beta,eva_k})
 \end{itemize}
 

\subsection{Experimental Settings}

\noindent
\textbf{Datasets.}
We have conducted experiments on three widely adopted public datasets: MNIST \citep{deng2012mnist}, CIFAR10 \citep{krizhevsky2009learning}, and CelebA \citep{liu2018large}, offering a range of objective categories with varying levels of learning complexity. We present detailed statistics of all datasets and how do we use them in \Cref{datasets_appendix}.

\noindent
\textbf{Models.}
We select three model architectures of different sizes in our experiments: a 7-layer convolutional neural network (CNN), a 5-layer multi-layer perceptron (MLP), and ResNet18. For the MNIST dataset, we employ two MLPs to form the GeoIB model (one as the compression encoder and one as the task decoder). For CIFAR10 and CelebA, we employ the ResNet18 as the encoder and one MLP  and CNN as decoder for CIFAR10 and CelebA respectively.

\noindent
\textbf{Metrics.}
Following the representative IB studies \citep{alemi2016deep,yang2025structured,zhai2022adversarial}, we quantify model utility by top-1 classification \textbf{Accuracy} on the held-out test set. Information compression is assessed by the mutual information  \textbf{ \textit{I(Z;X)} } between the learned representation and the input, estimated with MINE \citep{belghazi2018mutual}. To probe the leakage contained in $Z$, we perform two standard representation-level attacks: (i) a model inversion attack that reconstructs $x$ from $z$ \citep{fredrikson2015model}, evaluated by mean squared error ( \textbf{MSE}; lower MSE indicates stronger leakage); and (ii) a membership inference attack \citep{shokri2017membership}, evaluated by the membership inference accuracy (\textbf{MIA}) (higher values indicate stronger leakage). We also evaluate the efficiency of methods by recording their running time (\textbf{RT}).

\noindent
\textbf{Compared IB Benchmarks.}
We compare GeoIB against four representative Information Bottleneck (IB) variants: 
(1) the standard Variational IB (\textbf{VIB}) \citep{alemi2016deep}; 
(2) an IB variant where the mutual information term is estimated with \textbf{MINE} \citep{belghazi2018mutual}; 
(3) the state-of-the-art Structured IB (\textbf{SIB}) focusing on structure-aware feature learning \citep{yang2025structured}; and 
(4) the Adversarial IB (\textbf{AIB}) that incorporates adversarial regularization into the bottleneck \citep{zhai2022adversarial}. 
For fairness, all methods use the same backbone, data preprocessing, and training schedule; hyperparameters are tuned on the validation set following the original papers where applicable.

\subsection{Overall Evaluation of GeoIB} \label{overall_eva}

\noindent\textbf{Setup.}
We compare GeoIB with four representative IB variants on MNIST, CIFAR10, and CelebA under the same backbone and schedule. We fix the Lagrange multiplier $\beta=10^{-4}$ and the representation dimensionality $K=128$ for an apples-to-apples comparison.
We report top-1 accuracy (higher is better), the estimated mutual information $I(X;Z)$ via MINE (lower $I(X;Z)$ value indicates stronger compression), and model-inversion MSE from reconstructing $x$ from $z$ (higher indicates less leakage) in \Cref{tab_total}.

\noindent\textbf{Results.}
Across the three datasets, GeoIB attains the best or second-best results on all metrics. On MNIST, GeoIB achieves the highest accuracy and a (tied) lowest $I(X;Z)$, while matching the top inversion MSE. This suggests strong compression without sacrificing utility. On CIFAR10, GeoIB delivers the best accuracy and the highest inversion MSE (least leakage). Although SIB attains the lowest $I(X;Z)$, GeoIB offers a better accuracy–privacy trade-off overall. On CelebA, SIB slightly leads in accuracy and MI. GeoIB ranks second with competitive accuracy and privacy (indicated by MSE), confirming robustness on a more fine-grained, structured dataset. 

From the efficiency perspective, VIB achieves the best efficiency performance among all datasets and IB with MINE is the second best. They ignore the geometric information for mutual information estimation, enabling efficient training. GeoIB always consumes more running time, but it is similar to the RT of SIB, which considers the structure information.

\begin{figure*}[t] 
	\centering
	\subfloat{ 	\label{fig:mnistmodelaccixz}
		\includegraphics[scale=0.34]{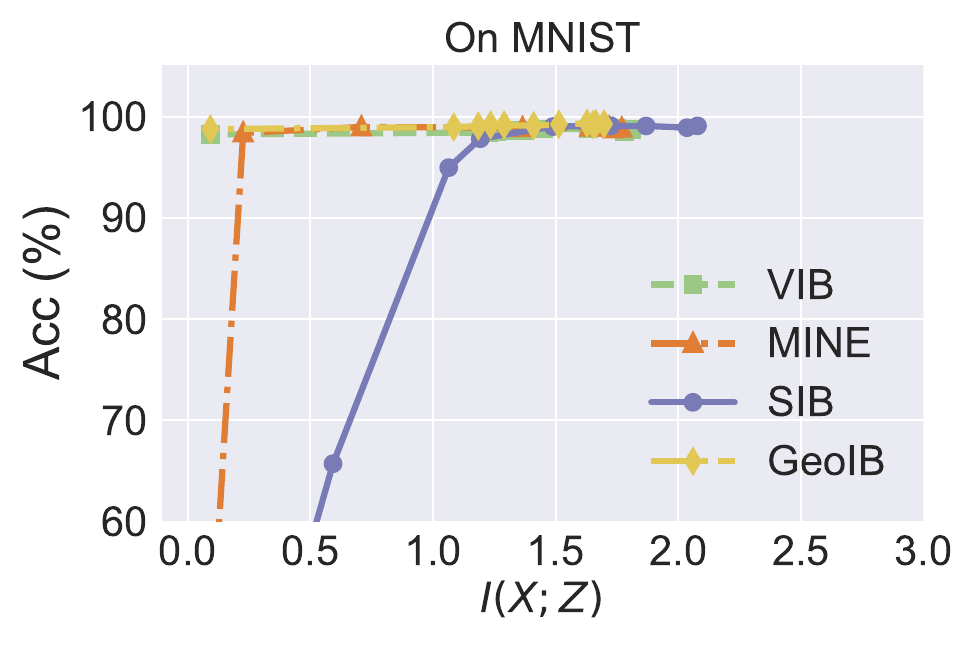}
	}			
	\subfloat{ 		\label{fig:cifar10modelaccixz}
		\includegraphics[scale=0.34]{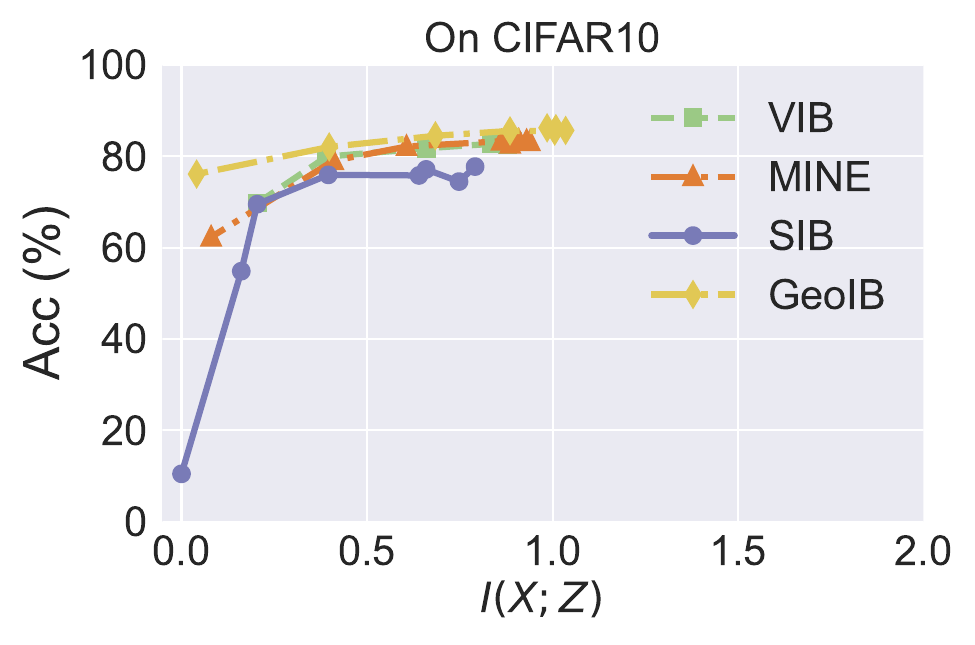}
	}
	\subfloat{  	\label{fig:celebamodelaccixz}
		\includegraphics[scale=0.34]{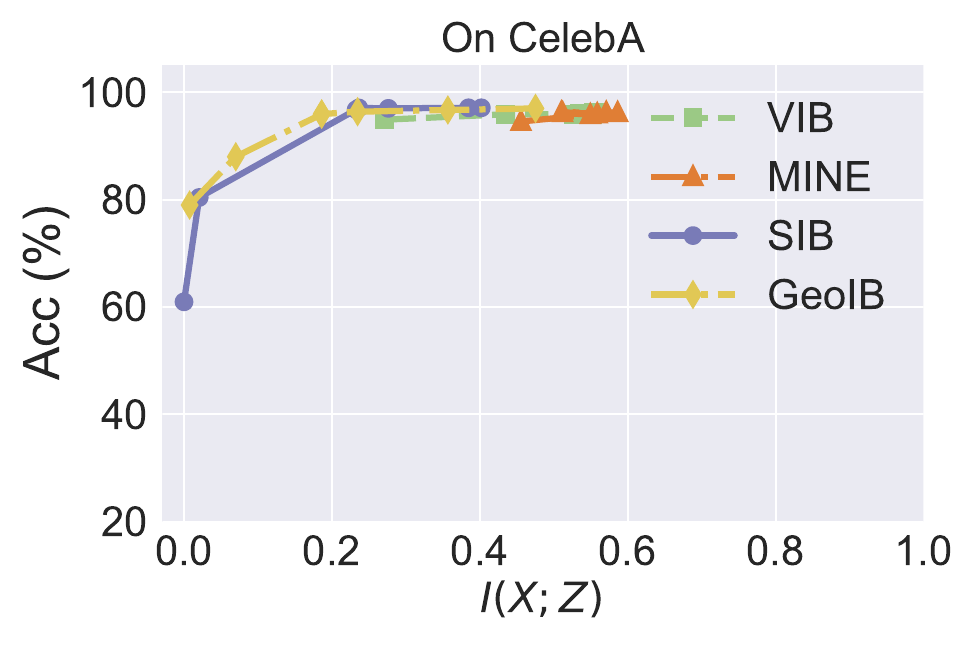}
	}  
	\vspace{-4mm}
	\caption{Evaluation of compression ratio and prediction accuracy from the information plane. } 
	\label{evaluation_of_general_acc_mi} 
	\vspace{-4mm}
\end{figure*}

\begin{figure*}[t] 
	\centering
	\subfloat{ 	 	\label{fig:mnistmodelaccbeta} \rotatebox{90}{ \hspace{10mm}	\small{ On MNIST} }
		\includegraphics[scale=0.3]{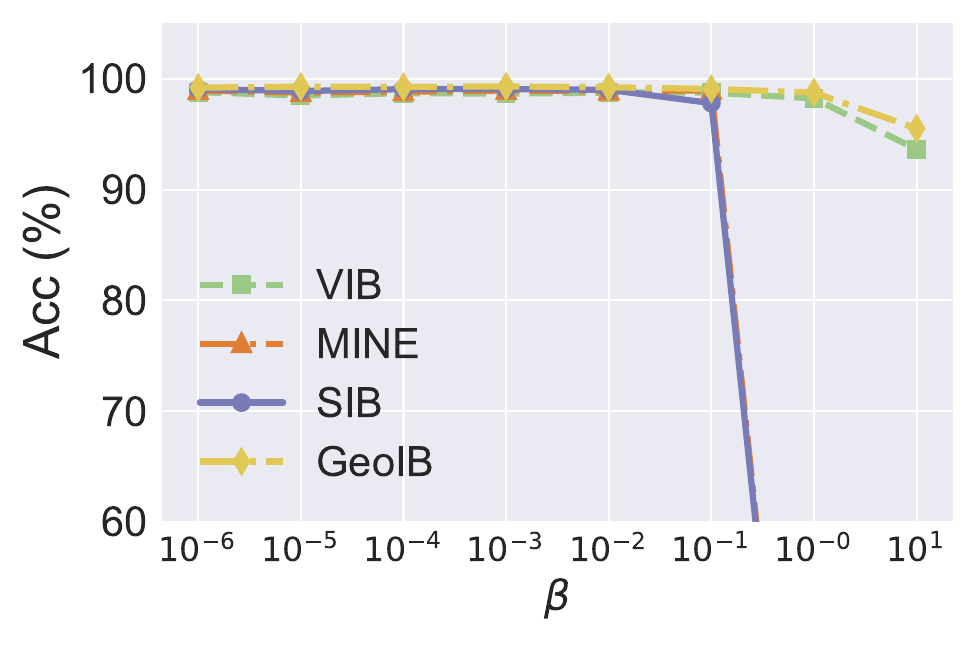}
	}			
	 	\hspace{3mm}
	\subfloat{ 		 	\label{fig:mnistmodelixzbeta}
		\includegraphics[scale=0.3]{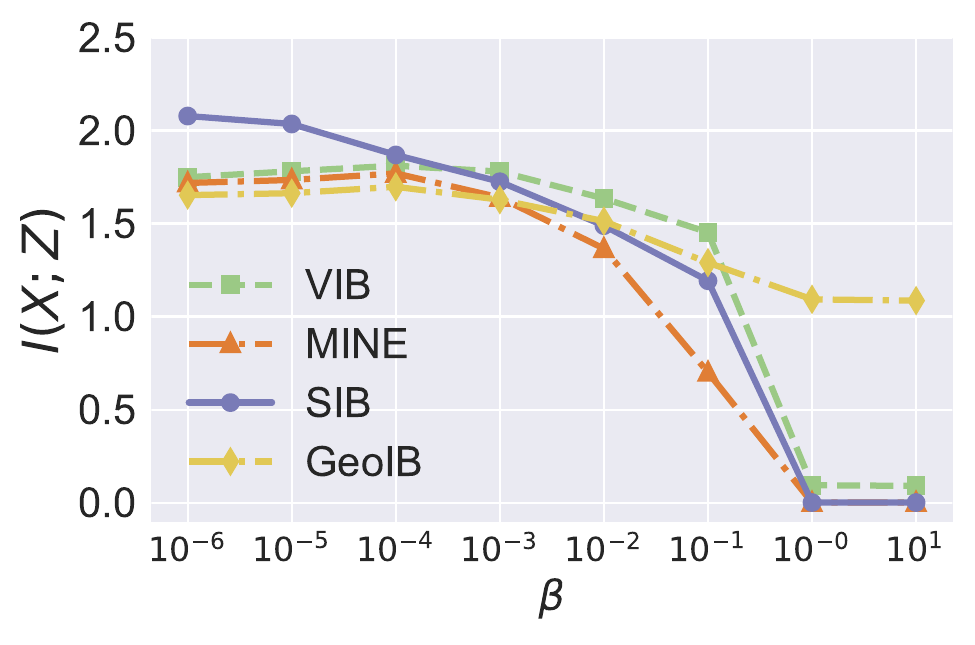}
	}
	 	\hspace{3mm}
	\subfloat{  	 	\label{fig:mnistmodelmsebeta}
		\includegraphics[scale=0.3]{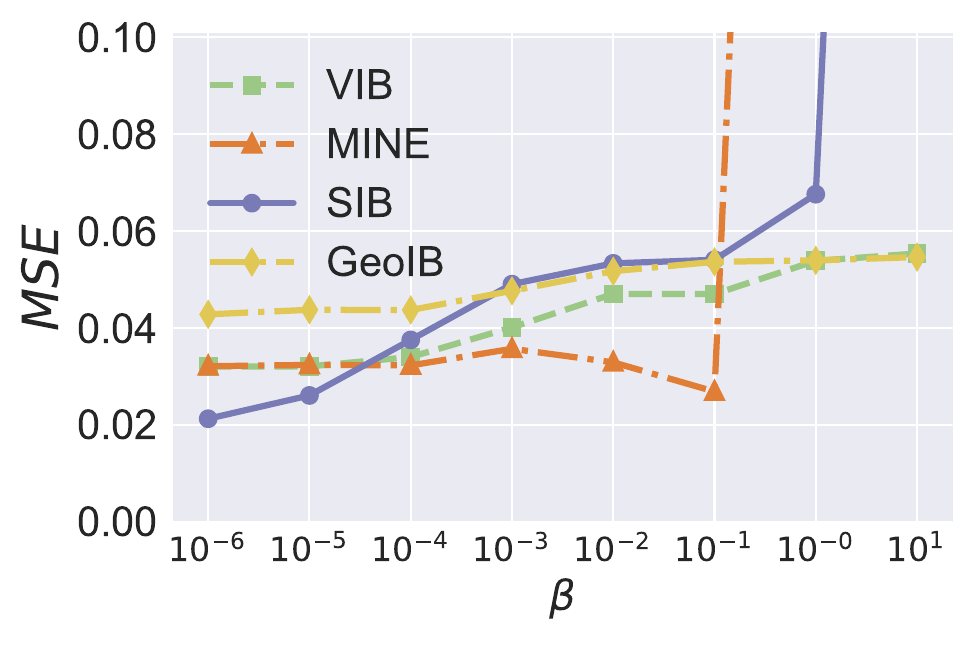}
	}  \\
	\vspace{-4mm}
	\subfloat{ 	\label{fig:cifar10modelaccbeta} \rotatebox{90}{ \hspace{10mm}	\small{ On CIFAR10} }
		\includegraphics[scale=0.3]{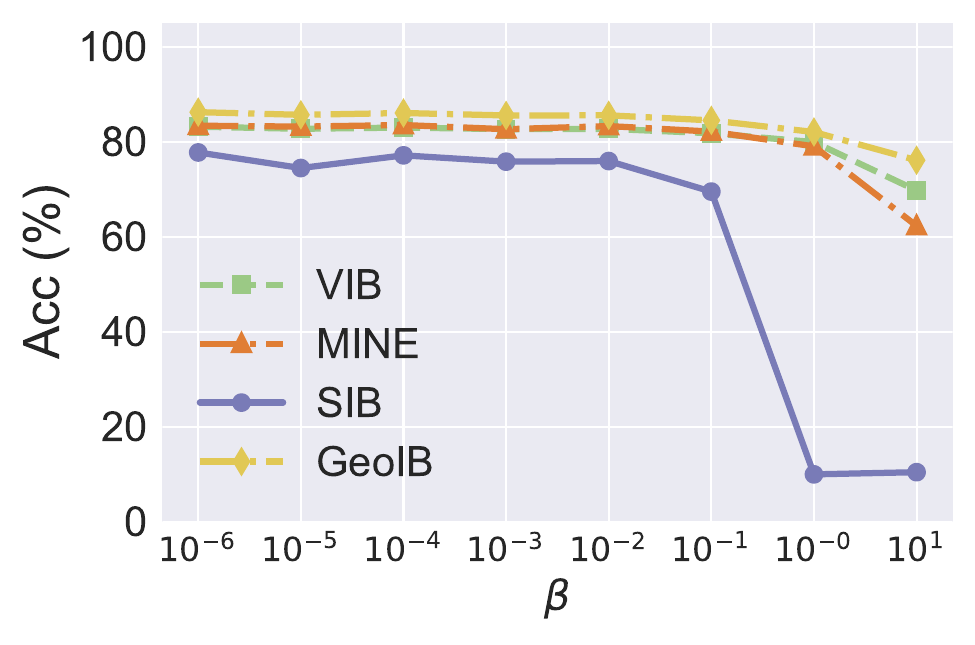}
	}			
	 	\hspace{3mm}
	\subfloat{ 			\label{fig:cifar10modelixzbeta}
		\includegraphics[scale=0.3]{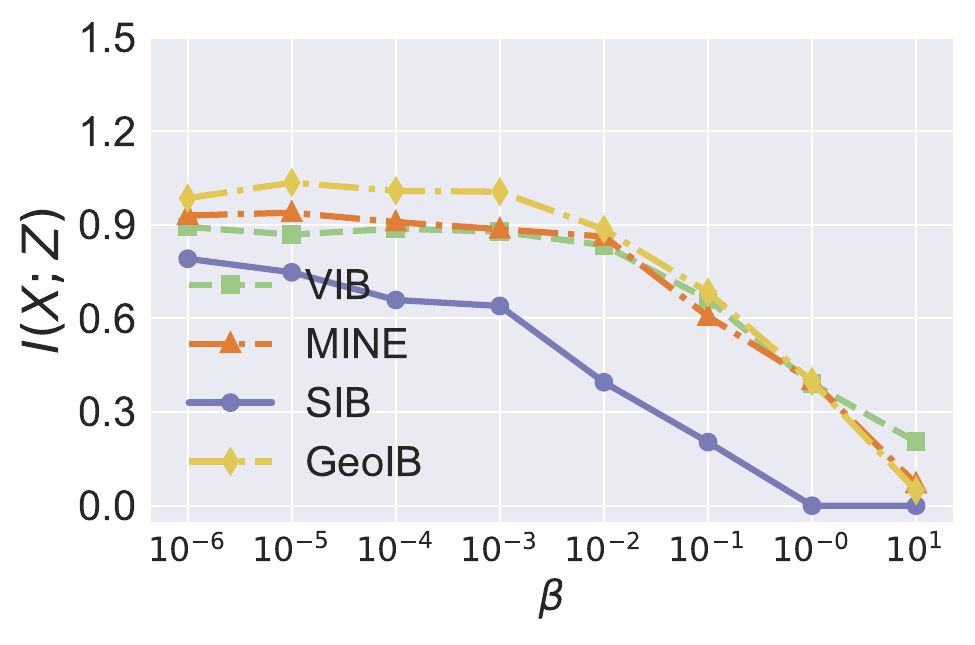}
	}
		\hspace{3mm}
	\subfloat{  	 	\label{fig:cifar10modelmsebeta}
		\includegraphics[scale=0.3]{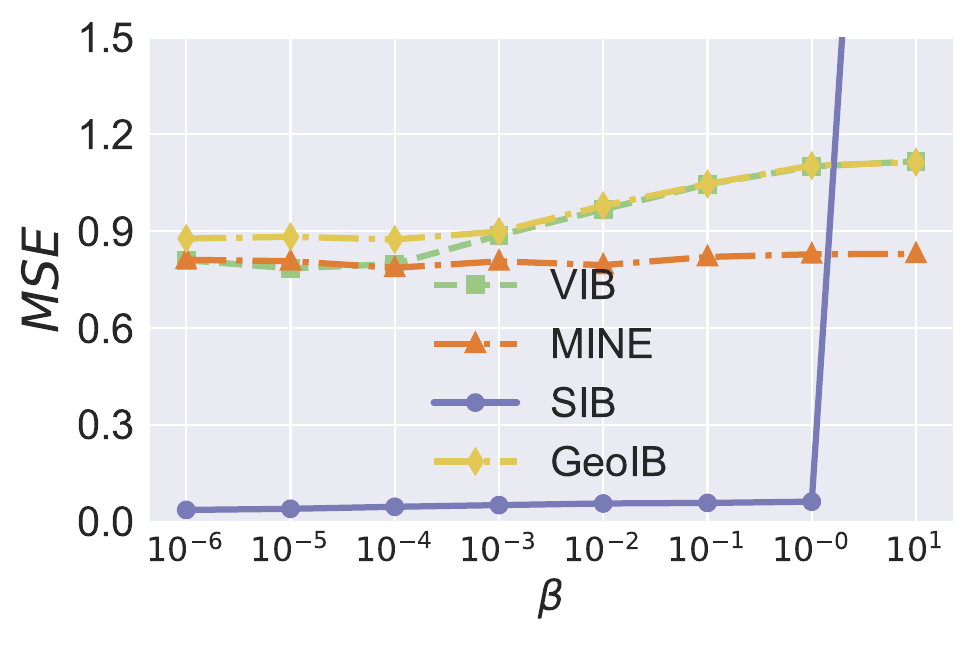}
	}  \\
	\vspace{-4mm}
	\subfloat{ 	\label{fig:celebamodelaccbeta} \rotatebox{90}{ \hspace{10mm}	\small{ On CelebA} }
		\includegraphics[scale=0.3]{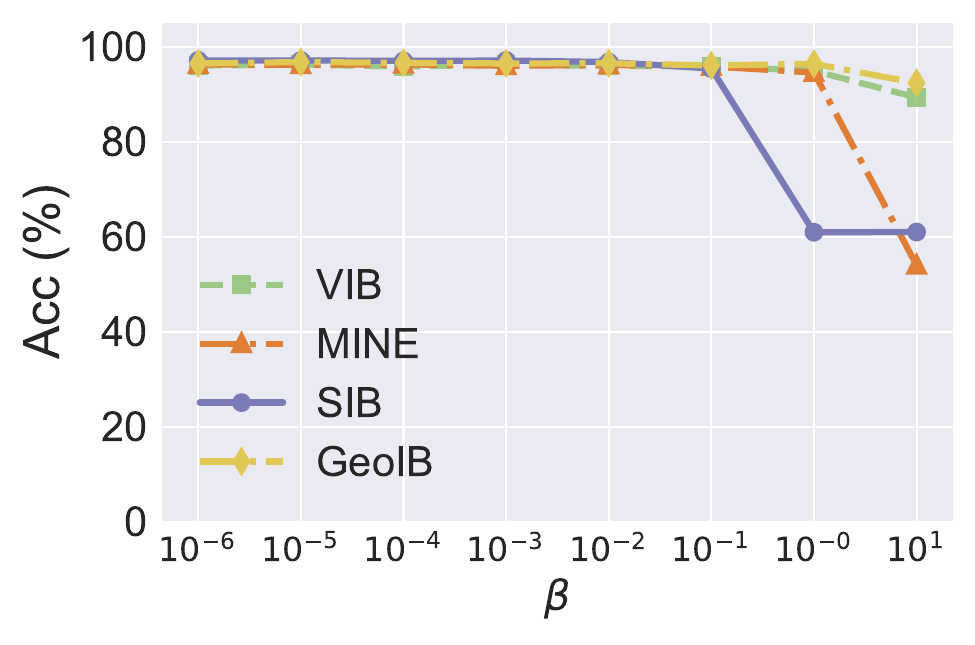}
	}			
		\hspace{3mm}
	\subfloat{ 			\label{fig:celebamodelixzbeta}
		\includegraphics[scale=0.3]{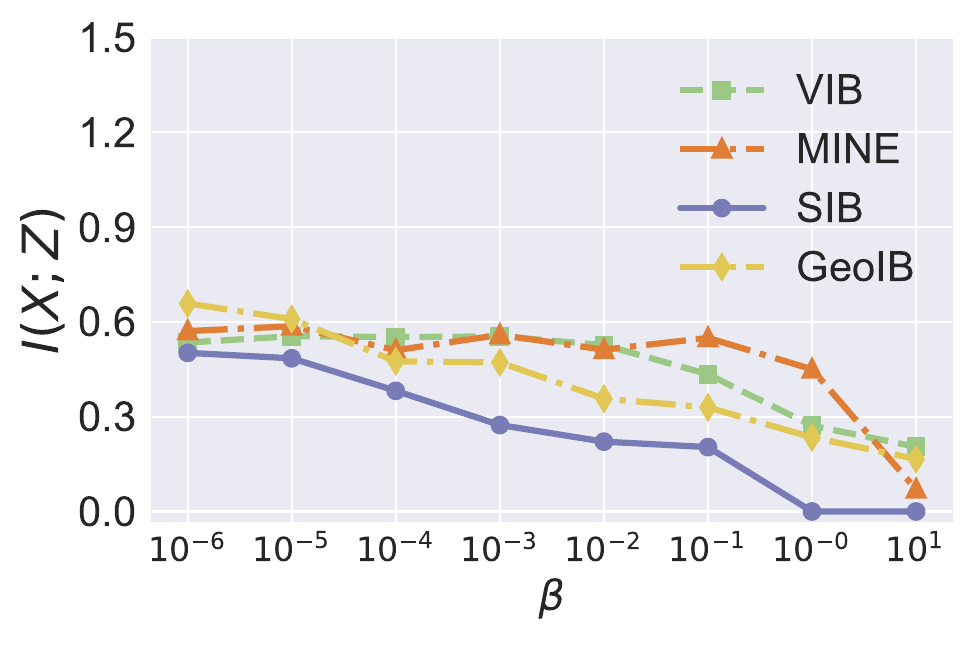}
	}
		\hspace{3mm}
	\subfloat{  		\label{fig:celebamodelmsebeta}
		\includegraphics[scale=0.3]{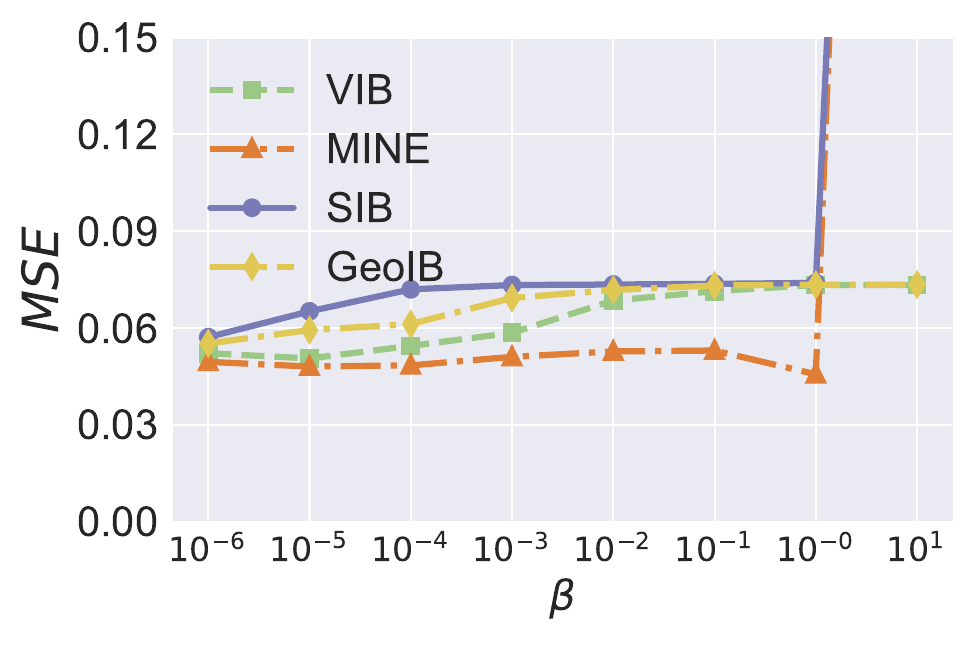}
	}  
	\vspace{-4mm}
	\caption{Evaluation about the impact of the Bottleneck multiplier $\beta$. } 
	\label{evaluation_of_general_beta} 
		\vspace{-4mm}
\end{figure*}

\begin{figure*}[t] 
	\centering
	\subfloat[\scriptsize $\beta=10^{-4}$, Accuracy $=99.28\%$ ]{  	\label{fig:repdemomnist00001}  \rotatebox{90}{ \hspace{10mm}	{On MNIST} }
		\includegraphics[scale=0.34]{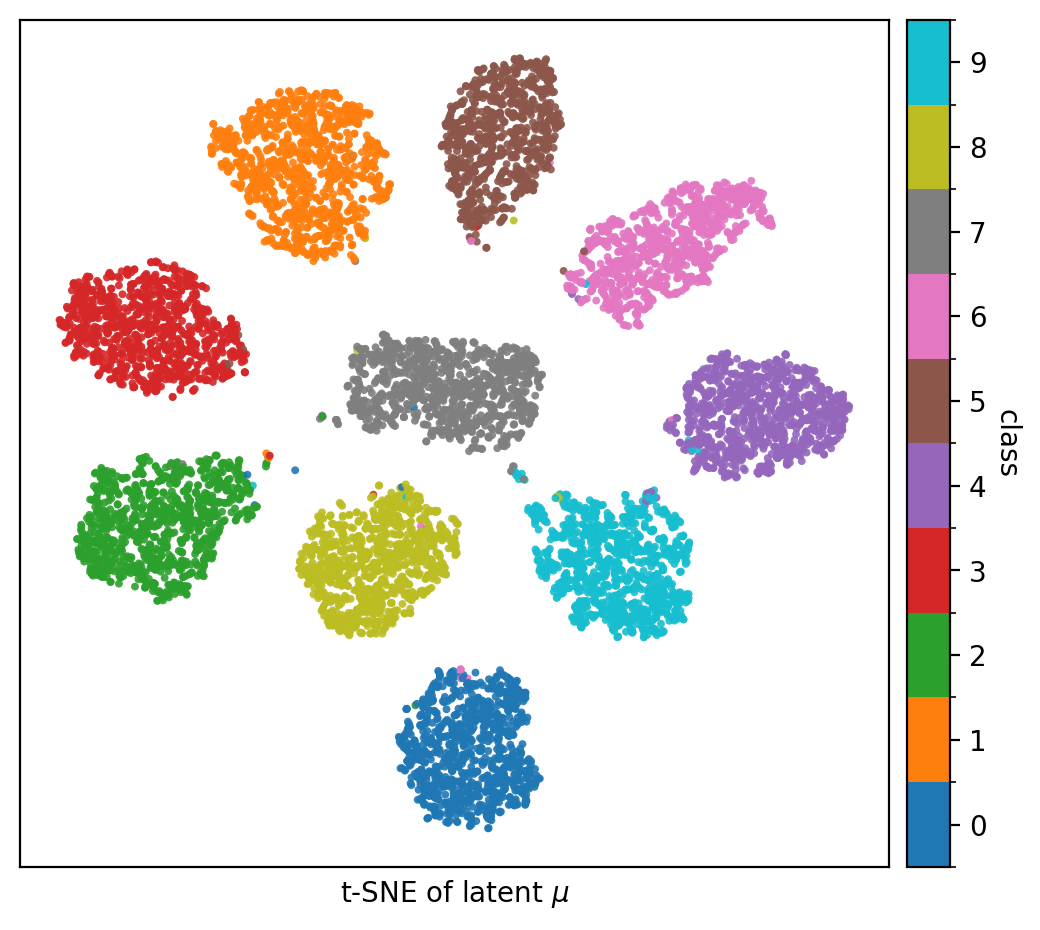}
	}			
		\hspace{3mm}
	\subfloat[\scriptsize $\beta=10^{0}$, Accuracy $=98.77\%$]{ 		\label{fig:repdemomnist1}
		\includegraphics[scale=0.34]{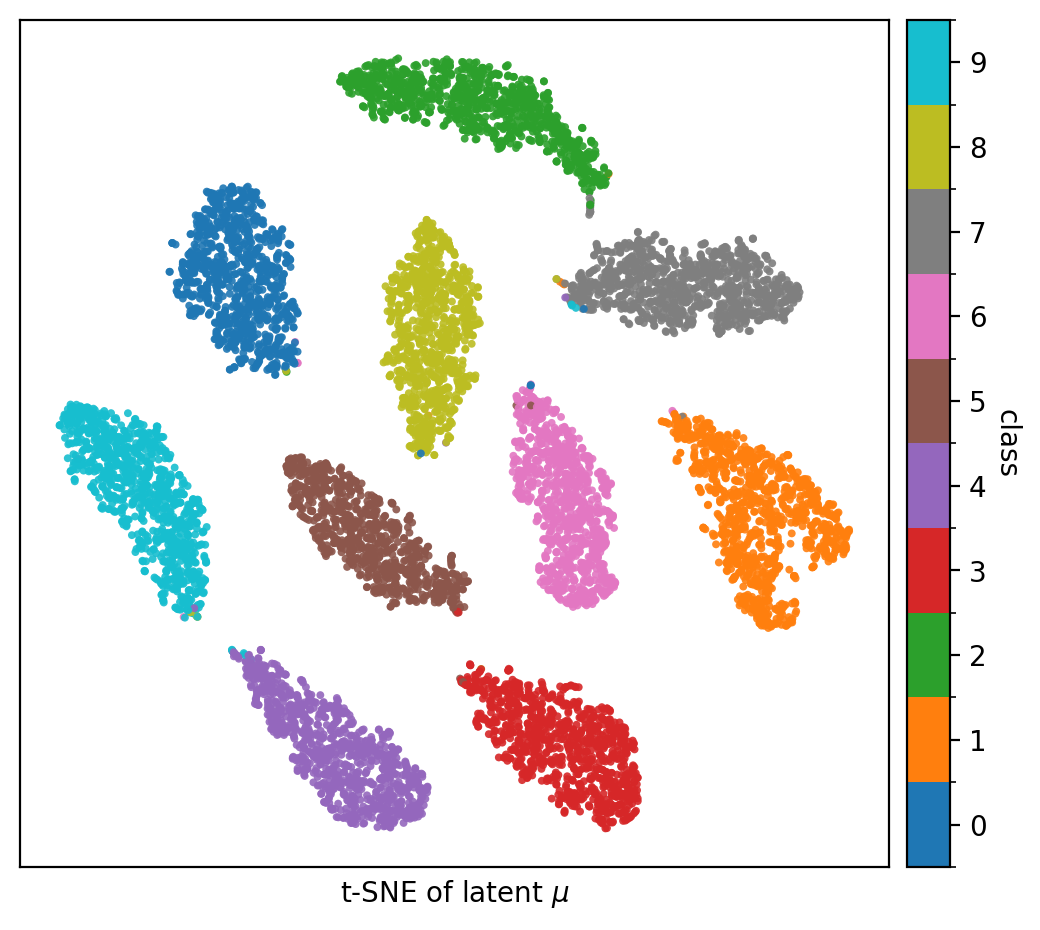}
	}
		\hspace{3mm}
	\subfloat[\scriptsize $\beta=10^{1}$, Accuracy $=95.53\%$]{  	\label{fig:repdemomnist10}
		\includegraphics[scale=0.34]{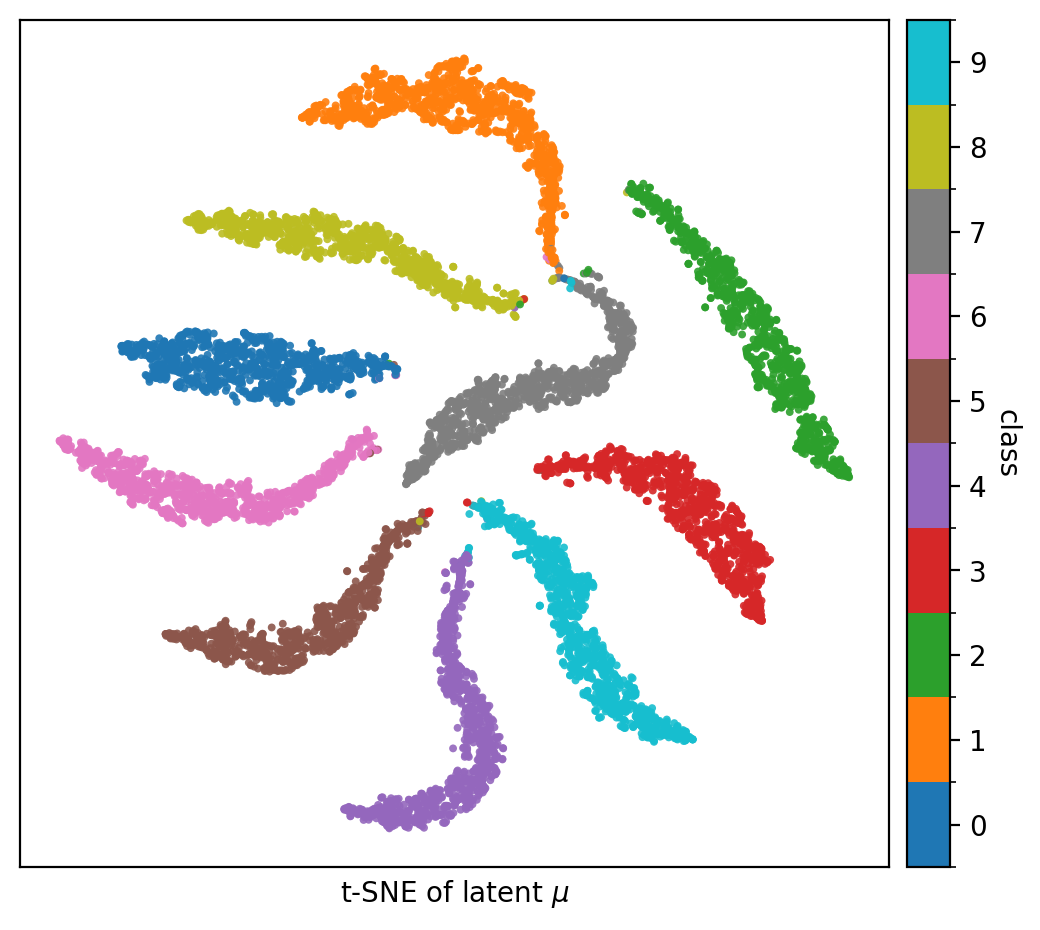}
	}  
	\vspace{-4mm}
	\caption{Visualizing representation embeddings of posterior means $\mu(x)$ for 10{,}000 test images in two dimensions on MNIST ($K=128$). Colors denote true labels. From left to right: $\beta=10^{-4}$, $10^{0}$, and $10^{1}$; the corresponding test accuracies are shown below each panel. As $\beta$ increases, within-class dispersion shrinks and clusters move toward class-wise prototypes, indicating stronger compression; accuracy decreases accordingly.
	} 
	\label{compressive_on_mnist} 
\end{figure*}

\begin{figure*}[t] 
	\centering
	\vspace{-6mm}
	\subfloat{ 	 	\label{fig:mnistmodelacck}  \rotatebox{90}{ \hspace{8mm}	\small{ On MNIST} }
		\includegraphics[scale=0.3]{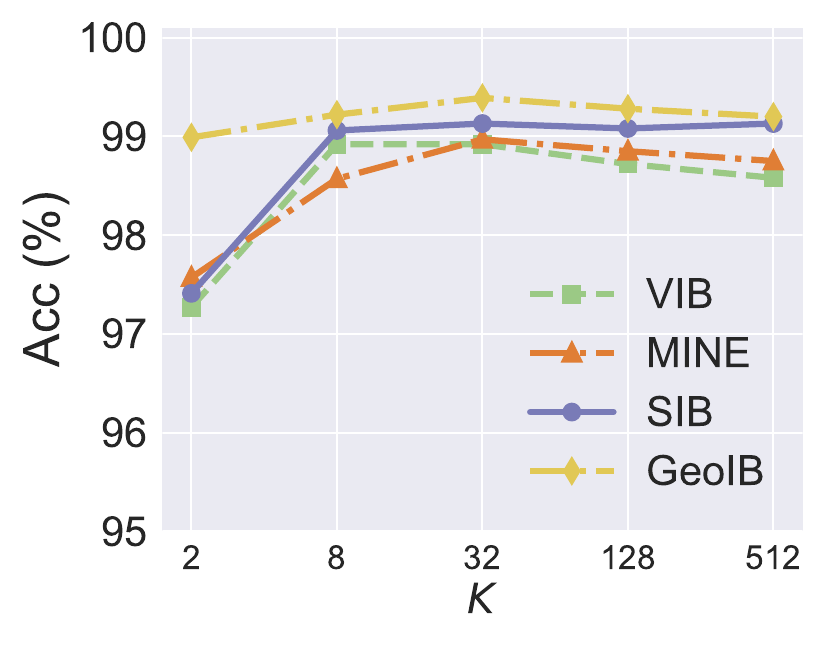}
	}			
	\subfloat{ 		 	\label{fig:mnistmodelixzk}
		\includegraphics[scale=0.3]{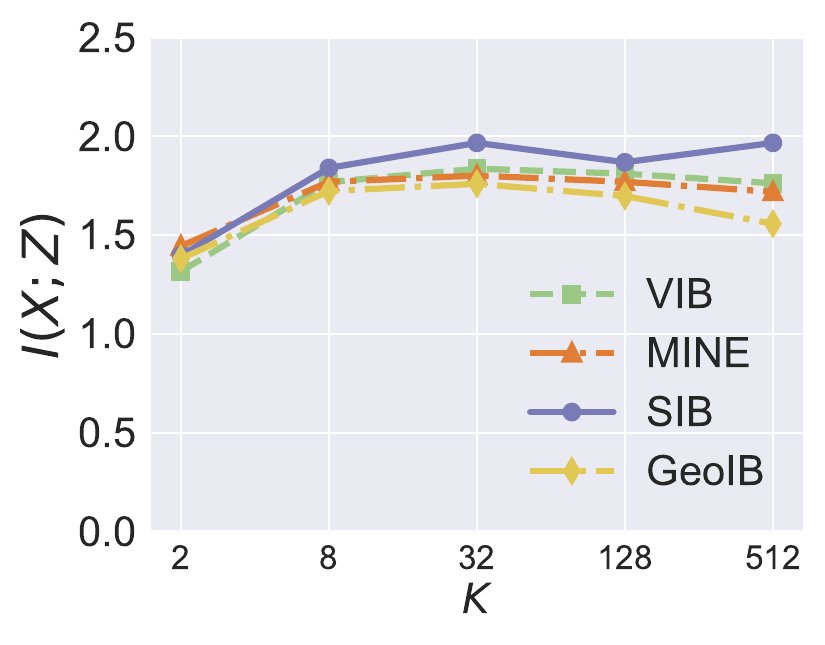}
	}
	\subfloat{  		\label{fig:mnistmodelmsek}
		\includegraphics[scale=0.3]{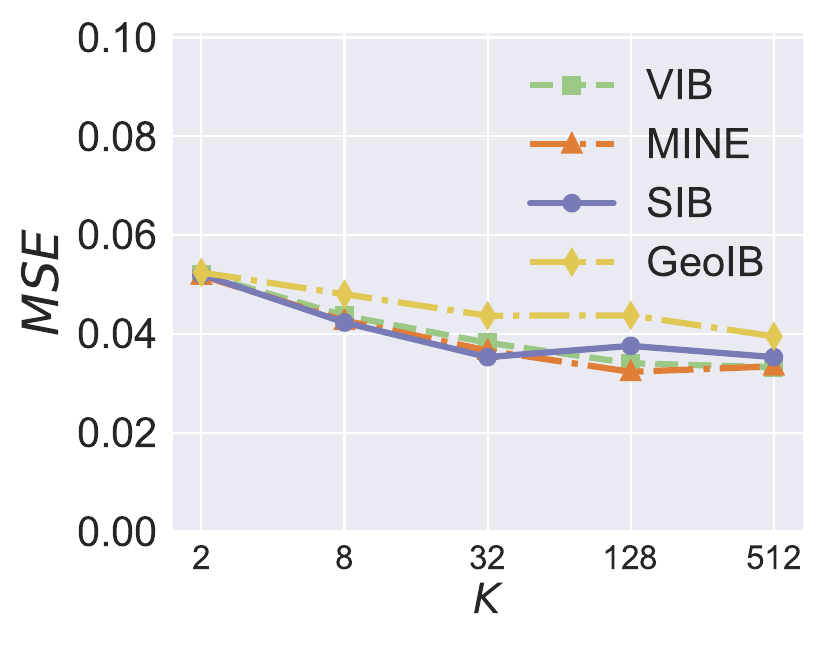}
	} 
	\subfloat{  	 	\label{fig:mnistmodelmiak}
		\includegraphics[scale=0.3]{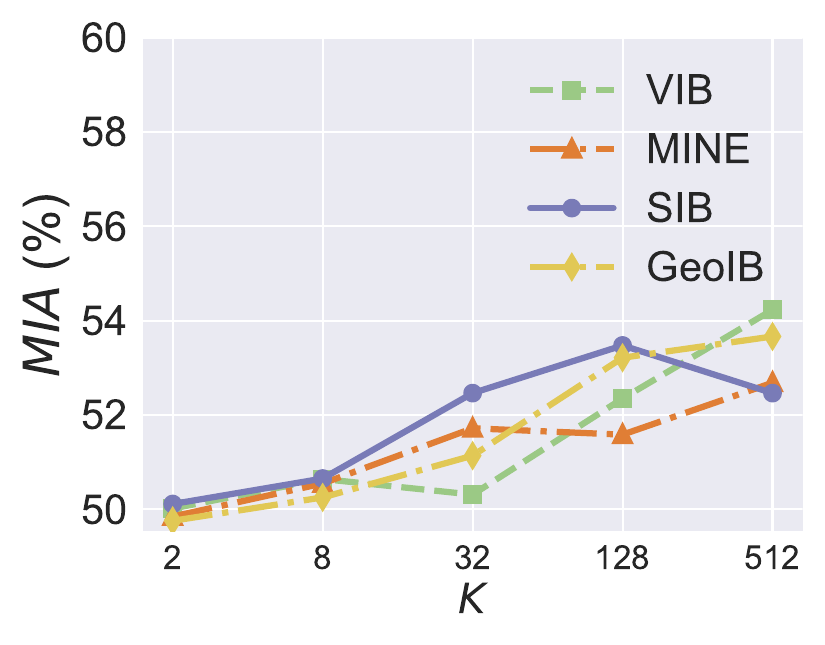}
	}
	\\
	\vspace{-4mm}
	\subfloat{ 	 	\label{fig:cifar10modelacck}  \rotatebox{90}{ \hspace{8mm}	\small{ On CIFAR10} }
		\includegraphics[scale=0.3]{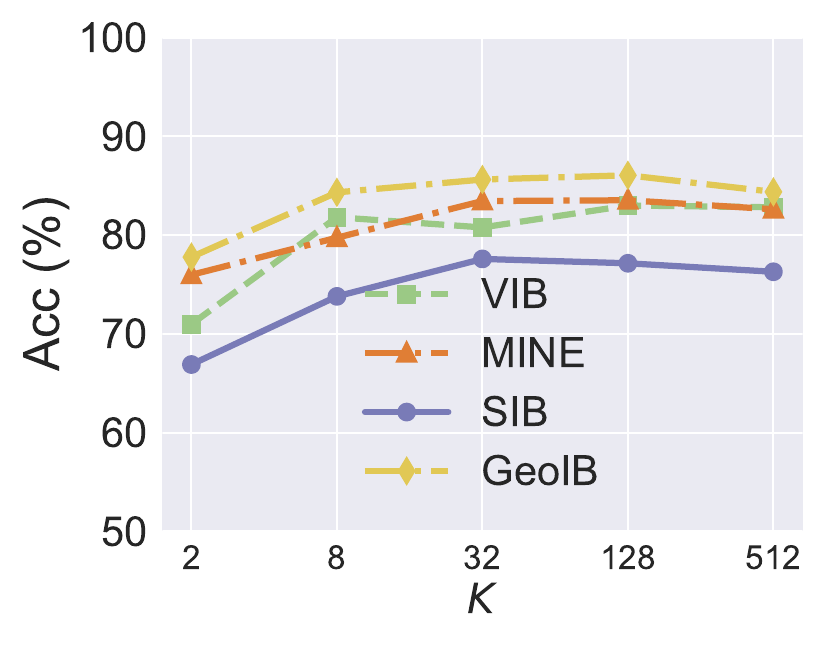}
	}			
	\subfloat{ 		 	\label{fig:cifar10modelixzk}
		\includegraphics[scale=0.3]{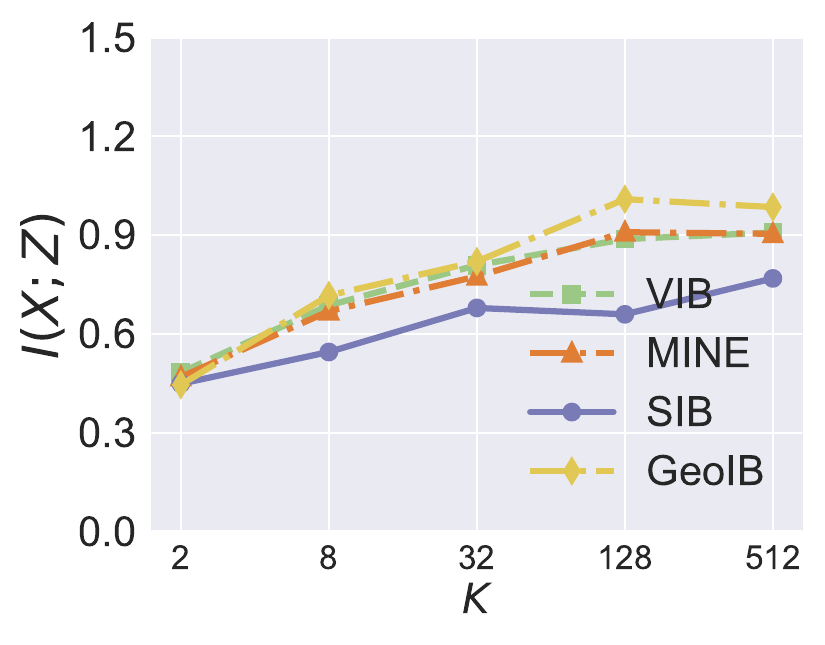}
	}
	\subfloat{  	 		\label{fig:cifar10modelmsek}
		\includegraphics[scale=0.3]{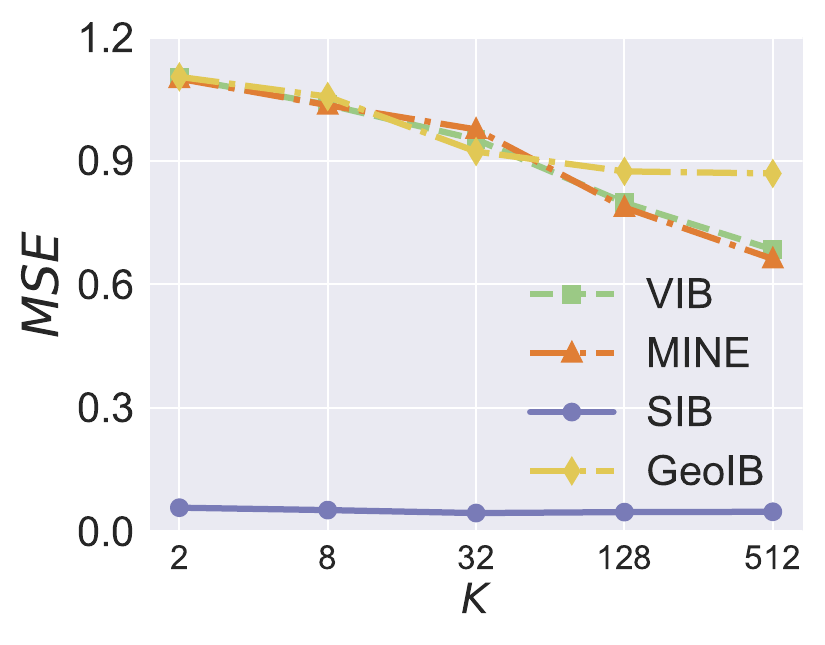}
	} 
	\subfloat{  	 	\label{fig:cifar10modelmiak}
		\includegraphics[scale=0.3]{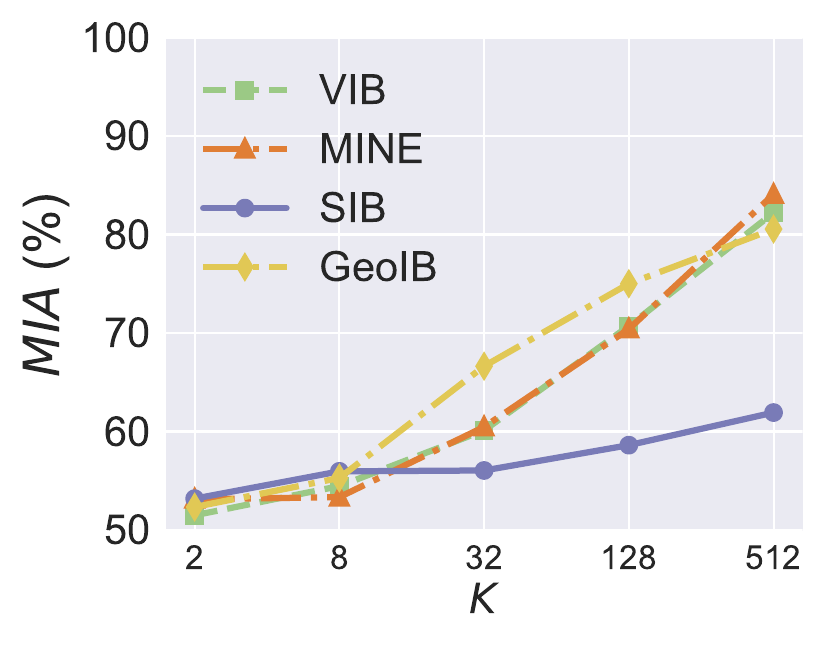}
	}  
	\vspace{-4mm}
	\caption{Evaluation about the impact of the Representation Dimensionality $K$.} 
	\label{evaluation_of_general_k} 
		\vspace{-4mm}
\end{figure*}

\subsection{Information-Plane Evaluation: Compression vs.\ Accuracy} \label{info_plane}

\noindent\textbf{Setup.}
We trace each method’s information plane on MNIST, CIFAR10, and CelebA by randomly selecting the Lagrange multiplier $\beta$ from $10^{-6}$ to $10^{1}$. During this experiment, we fixed representation size $K\!=\!128$. For every~$\beta$, we train to convergence and record test accuracy (higher is better) and the estimated mutual information $I(X;Z)$ via MINE as shown in \Cref{evaluation_of_general_acc_mi}.
 
\noindent
\textbf{Results.} In \Cref{evaluation_of_general_acc_mi}, across datasets, GeoIB's curve is consistently shifted up and left relative to the baselines, achieving higher accuracy at a matched $I(X;Z)$, or a smaller $I(X;Z)$ at a matched accuracy (Pareto improvement). When compression is weak, i.e., large $I(X;Z)$, all methods reach a high-accuracy plateau.
As compression strengthens, i.e., smaller $I(X;Z)$, VIB, MINE, and SIB exhibit clear accuracy drops, while GeoIB maintains accuracy over a wider low-$I(X;Z)$ range before degrading.
On CIFAR10 the gap is most visible, where GeoIB preserves accuracy at lower $I(X;Z)$, And on CelebA, the curves cluster near the top but GeoIB attains comparable accuracy with less information in $Z$.


\subsection{Ablation: Influence of the Bottleneck multiplier $\beta$} \label{eva_beta}

\noindent\textbf{Setup.}
We study the impact of the Bottleneck multiplier by sweeping $\beta$ on a logarithmic grid from $10^{-6}$ to $10^{1}$, with representation size fixed at $K=128$. 
For each $\beta$, we train a model and report test accuracy (higher is better), the estimated mutual information $I(X;Z)$ (lower means stronger compression), and the model-inversion MSE from reconstructing $x$ from $z$ (higher means less leakage). Results for MNIST, CIFAR10, and CelebA are shown in \Cref{evaluation_of_general_beta}.

\noindent\textbf{Results.}
(Left column) \emph{Accuracy vs.\ $\beta$.} GeoIB maintains the highest or near-highest accuracy over a wide range of $\beta$ on all datasets. 
Baselines, especially MINE and SIB, exhibit sharp degradation when $\beta\!\ge\!10^{-1}$; on CIFAR10 and CelebA, SIB collapses at large $\beta$.

(Middle) \emph{$I(X;Z)$ vs.\ $\beta$.} All methods show the expected monotonic decrease of $I(X;Z)$ as $\beta$ increases. 
SIB achieves the smallest $I(X;Z)$ (strongest compression) but at the cost of the pronounced accuracy drop above, whereas GeoIB attains competitive compression while preserving accuracy.

(Right) \emph{MSE vs.\ $\beta$.} GeoIB yields the best or second-best MSE across most $\beta$, indicating stronger resistance to inversion attacks. 
At extreme $\beta$ ($10^{0}\!-\!10^{1}$), the MSE of some baselines spikes due to representation collapse. Reconstructions become effectively random, which inflates MSE but coincides with poor utility.
 

\noindent\textbf{Qualitative visualization.}
To examine how the bottleneck strength shapes the representation, we visualize a 2-D t-SNE of the representation $\mu(x)=\mathbb{E}[ Z \mid x]$ for 10{,}000 MNIST test images at three values of $\beta$ (\Cref{compressive_on_mnist}); the CIFAR10 results are in \Cref{compressive_on_cifar10} in \Cref{addtional_exp}. 
With a small Bottleneck multiplier ($\beta=10^{-4}$), clusters are well separated and exhibit relatively large within-class spread, indicating that $Z$ retains fine-grained input details. At $\beta=10^{0}$, representation clusters contract toward class-wise prototypes while remaining separable.
At $\beta=10^{1}$, representation embeddings concentrate along narrow arcs near the class centers, consistent with stronger compression and the accuracy drop observed in \Cref{evaluation_of_general_beta}.



\subsection{Ablation: Influence of the Representation Dimensionality $K$} \label{eva_k}

\noindent\textbf{Setup.}
We vary the representation dimensionality $K$ from $2^1$ to $2^9$ with the Lagrange multiplier fixed at $\beta=10^{-4}$. 
For each $K$, we report the accuracy, mutual information $I(X;Z)$, MSE, and MIA for different IB methods. Results for MNIST and CIFAR10 are shown in \Cref{evaluation_of_general_k}.

\noindent\textbf{Results.}
In the Left column of \Cref{evaluation_of_general_k}, accuracy improves with $K$ and then saturates (MNIST: gains plateau around $K\!\ge\!32$; CIFAR10: around $K\!\ge\!128$). 
GeoIB attains the highest or near-highest accuracy across $K$, especially in the low-to-mid range. In the Mid-left column, as expected, $I(X;Z)$ increases with $K$, indicating weaker compression when the latent space is wider.  Tail non-monotonicity on MNIST at very large $K$ is minor and likely due to the representation dimensionality of $K=32$ is already large enough for MNIST dataset. The Mid-right column shows that MSE generally decreases as $K$ grows (reconstructions become easier), reflecting increased leakage with higher-dimensional $Z$.  GeoIB keeps MSE competitively high (i.e., more resistant to inversion) at small and medium $K$. In the Right column, MIA success rises with $K$ on both datasets, corroborating that larger $Z$ carries more membership signal. Across a broad range of $K$, GeoIB remains competitive; when $K$ is small-to-medium, it attains a favorable utility–privacy balance.

\section{Summary and Future Work} \label{summary}

We propose the GeoIB to solve the IB problem from a information geometry perspective. The key ingredients are (i) a distribution-level Fisher--Rao (FR) discrepancy that locally matches KL to second order and is invariant under smooth reparameterizations of the latent, and (ii) a geometry-level Jacobian–Frobenius (JF) penalty that provides a local capacity-type upper bound on $I_\phi(Z;X)$ by discouraging pullback volume expansion. 

 
Promising directions include tightening the FR-based proxies beyond the local (second-order) regime and replacing the trace relaxation with sharper spectral approximations; improving the efficiency of GeoIB training; and extending GeoIB to federated and privacy-preserving settings where FR and JF controls might yield verifiable unlearning guarantees.



\begin{acks}
	The authors would like to thank the anonymous reviewers for their valuable comments and review. 
\end{acks}

\bibliographystyle{ACM-Reference-Format}
\bibliography{GIB_ref}


\begin{thebibliography}{41}


\ifx \showCODEN    \undefined \def \showCODEN     #1{\unskip}     \fi
\ifx \showISBNx    \undefined \def \showISBNx     #1{\unskip}     \fi
\ifx \showISBNxiii \undefined \def \showISBNxiii  #1{\unskip}     \fi
\ifx \showISSN     \undefined \def \showISSN      #1{\unskip}     \fi
\ifx \showLCCN     \undefined \def \showLCCN      #1{\unskip}     \fi
\ifx \shownote     \undefined \def \shownote      #1{#1}          \fi
\ifx \showarticletitle \undefined \def \showarticletitle #1{#1}   \fi
\ifx \showURL      \undefined \def \showURL       {\relax}        \fi
\providecommand\bibfield[2]{#2}
\providecommand\bibinfo[2]{#2}
\providecommand\natexlab[1]{#1}
\providecommand\showeprint[2][]{arXiv:#2}

\bibitem[Achille and Soatto(2018)]%
        {achille2018information}
\bibfield{author}{\bibinfo{person}{Alessandro Achille} {and}
  \bibinfo{person}{Stefano Soatto}.} \bibinfo{year}{2018}\natexlab{}.
\newblock \showarticletitle{Information dropout: Learning optimal
  representations through noisy computation}.
\newblock \bibinfo{journal}{\emph{IEEE transactions on pattern analysis and
  machine intelligence}} \bibinfo{volume}{40}, \bibinfo{number}{12}
  (\bibinfo{year}{2018}), \bibinfo{pages}{2897--2905}.
\newblock


\bibitem[Alemi et~al\mbox{.}(2016)]%
        {alemi2016deep}
\bibfield{author}{\bibinfo{person}{Alexander~A Alemi}, \bibinfo{person}{Ian
  Fischer}, \bibinfo{person}{Joshua~V Dillon}, {and} \bibinfo{person}{Kevin
  Murphy}.} \bibinfo{year}{2016}\natexlab{}.
\newblock \showarticletitle{Deep variational information bottleneck}.
\newblock \bibinfo{journal}{\emph{arXiv preprint arXiv:1612.00410}}
  (\bibinfo{year}{2016}).
\newblock


\bibitem[Amari(1998)]%
        {amari1998natural}
\bibfield{author}{\bibinfo{person}{Shun-Ichi Amari}.}
  \bibinfo{year}{1998}\natexlab{}.
\newblock \showarticletitle{Natural gradient works efficiently in learning}.
\newblock \bibinfo{journal}{\emph{Neural computation}} \bibinfo{volume}{10},
  \bibinfo{number}{2} (\bibinfo{year}{1998}), \bibinfo{pages}{251--276}.
\newblock


\bibitem[Amari(2002)]%
        {amari2002information}
\bibfield{author}{\bibinfo{person}{S-I Amari}.}
  \bibinfo{year}{2002}\natexlab{}.
\newblock \showarticletitle{Information geometry on hierarchy of probability
  distributions}.
\newblock \bibinfo{journal}{\emph{IEEE transactions on information theory}}
  \bibinfo{volume}{47}, \bibinfo{number}{5} (\bibinfo{year}{2002}),
  \bibinfo{pages}{1701--1711}.
\newblock


\bibitem[Amari(2016)]%
        {amari2016information}
\bibfield{author}{\bibinfo{person}{Shun-ichi Amari}.}
  \bibinfo{year}{2016}\natexlab{}.
\newblock \bibinfo{booktitle}{\emph{Information geometry and its
  applications}}. Vol.~\bibinfo{volume}{194}.
\newblock \bibinfo{publisher}{Springer}.
\newblock


\bibitem[Amari and Nagaoka(2000)]%
        {amari2000methods}
\bibfield{author}{\bibinfo{person}{Shun-ichi Amari} {and}
  \bibinfo{person}{Hiroshi Nagaoka}.} \bibinfo{year}{2000}\natexlab{}.
\newblock \bibinfo{booktitle}{\emph{Methods of information geometry}}.
  Vol.~\bibinfo{volume}{191}.
\newblock \bibinfo{publisher}{American Mathematical Soc.}
\newblock


\bibitem[Belghazi et~al\mbox{.}(2018)]%
        {belghazi2018mutual}
\bibfield{author}{\bibinfo{person}{Mohamed~Ishmael Belghazi},
  \bibinfo{person}{Aristide Baratin}, \bibinfo{person}{Sai Rajeshwar},
  \bibinfo{person}{Sherjil Ozair}, \bibinfo{person}{Yoshua Bengio},
  \bibinfo{person}{Aaron Courville}, {and} \bibinfo{person}{Devon Hjelm}.}
  \bibinfo{year}{2018}\natexlab{}.
\newblock \showarticletitle{Mutual information neural estimation}. In
  \bibinfo{booktitle}{\emph{International conference on machine learning}}.
  PMLR, \bibinfo{pages}{531--540}.
\newblock


\bibitem[Butakov et~al\mbox{.}(2024)]%
        {butakov2024mutual}
\bibfield{author}{\bibinfo{person}{Ivan Butakov}, \bibinfo{person}{Aleksandr
  Tolmachev}, \bibinfo{person}{Sofia Malanchuk}, \bibinfo{person}{Anna
  Neopryatnaya}, {and} \bibinfo{person}{Alexey Frolov}.}
  \bibinfo{year}{2024}\natexlab{}.
\newblock \showarticletitle{Mutual information estimation via normalizing
  flows}.
\newblock \bibinfo{journal}{\emph{Advances in Neural Information Processing
  Systems}}  \bibinfo{volume}{37} (\bibinfo{year}{2024}),
  \bibinfo{pages}{3027--3057}.
\newblock


\bibitem[Cheng et~al\mbox{.}(2020)]%
        {cheng2020club}
\bibfield{author}{\bibinfo{person}{Pengyu Cheng}, \bibinfo{person}{Weituo Hao},
  \bibinfo{person}{Shuyang Dai}, \bibinfo{person}{Jiachang Liu},
  \bibinfo{person}{Zhe Gan}, {and} \bibinfo{person}{Lawrence Carin}.}
  \bibinfo{year}{2020}\natexlab{}.
\newblock \showarticletitle{Club: A contrastive log-ratio upper bound of mutual
  information}. In \bibinfo{booktitle}{\emph{International conference on
  machine learning}}. PMLR, \bibinfo{pages}{1779--1788}.
\newblock


\bibitem[Deng(2012)]%
        {deng2012mnist}
\bibfield{author}{\bibinfo{person}{Li Deng}.} \bibinfo{year}{2012}\natexlab{}.
\newblock \showarticletitle{The mnist database of handwritten digit images for
  machine learning research [best of the web]}.
\newblock \bibinfo{journal}{\emph{IEEE signal processing magazine}}
  \bibinfo{volume}{29}, \bibinfo{number}{6} (\bibinfo{year}{2012}),
  \bibinfo{pages}{141--142}.
\newblock


\bibitem[Dubois et~al\mbox{.}(2021)]%
        {dubois2021lossy}
\bibfield{author}{\bibinfo{person}{Yann Dubois}, \bibinfo{person}{Benjamin
  Bloem-Reddy}, \bibinfo{person}{Karen Ullrich}, {and} \bibinfo{person}{Chris~J
  Maddison}.} \bibinfo{year}{2021}\natexlab{}.
\newblock \showarticletitle{Lossy compression for lossless prediction}.
\newblock \bibinfo{journal}{\emph{Advances in Neural Information Processing
  Systems}}  \bibinfo{volume}{34} (\bibinfo{year}{2021}),
  \bibinfo{pages}{14014--14028}.
\newblock


\bibitem[Franzese et~al\mbox{.}(2023)]%
        {franzese2023minde}
\bibfield{author}{\bibinfo{person}{Giulio Franzese}, \bibinfo{person}{Mustapha
  Bounoua}, {and} \bibinfo{person}{Pietro Michiardi}.}
  \bibinfo{year}{2023}\natexlab{}.
\newblock \showarticletitle{MINDE: Mutual information neural diffusion
  estimation}.
\newblock \bibinfo{journal}{\emph{arXiv preprint arXiv:2310.09031}}
  (\bibinfo{year}{2023}).
\newblock


\bibitem[Fredrikson et~al\mbox{.}(2015)]%
        {fredrikson2015model}
\bibfield{author}{\bibinfo{person}{Matt Fredrikson}, \bibinfo{person}{Somesh
  Jha}, {and} \bibinfo{person}{Thomas Ristenpart}.}
  \bibinfo{year}{2015}\natexlab{}.
\newblock \showarticletitle{Model inversion attacks that exploit confidence
  information and basic countermeasures}. In
  \bibinfo{booktitle}{\emph{Proceedings of the 22nd ACM SIGSAC conference on
  computer and communications security}}. \bibinfo{pages}{1322--1333}.
\newblock


\bibitem[Hong et~al\mbox{.}(2025)]%
        {hong2025comprehensive}
\bibfield{author}{\bibinfo{person}{Jung-Ho Hong}, \bibinfo{person}{Ho-Joong
  Kim}, \bibinfo{person}{Kyu-Sung Jeon}, {and} \bibinfo{person}{Seong-Whan
  Lee}.} \bibinfo{year}{2025}\natexlab{}.
\newblock \showarticletitle{Comprehensive Information Bottleneck for Unveiling
  Universal Attribution to Interpret Vision Transformers}. In
  \bibinfo{booktitle}{\emph{Proceedings of the Computer Vision and Pattern
  Recognition Conference}}. \bibinfo{pages}{25166--25175}.
\newblock


\bibitem[Hu et~al\mbox{.}(2024)]%
        {hu2024structured}
\bibfield{author}{\bibinfo{person}{Dou Hu}, \bibinfo{person}{Lingwei Wei},
  \bibinfo{person}{Yaxin Liu}, \bibinfo{person}{Wei Zhou}, {and}
  \bibinfo{person}{Songlin Hu}.} \bibinfo{year}{2024}\natexlab{}.
\newblock \showarticletitle{Structured probabilistic coding}. In
  \bibinfo{booktitle}{\emph{Proceedings of the AAAI Conference on Artificial
  Intelligence}}, Vol.~\bibinfo{volume}{38}. \bibinfo{pages}{12491--12501}.
\newblock


\bibitem[Kallenberg(1997)]%
        {kallenberg1997foundations}
\bibfield{author}{\bibinfo{person}{Olav Kallenberg}.}
  \bibinfo{year}{1997}\natexlab{}.
\newblock \bibinfo{booktitle}{\emph{Foundations of modern probability}}.
\newblock \bibinfo{publisher}{Springer}.
\newblock


\bibitem[Kolchinsky et~al\mbox{.}(2019)]%
        {kolchinsky2019nonlinear}
\bibfield{author}{\bibinfo{person}{Artemy Kolchinsky},
  \bibinfo{person}{Brendan~D Tracey}, {and} \bibinfo{person}{David~H Wolpert}.}
  \bibinfo{year}{2019}\natexlab{}.
\newblock \showarticletitle{Nonlinear information bottleneck}.
\newblock \bibinfo{journal}{\emph{Entropy}} \bibinfo{volume}{21},
  \bibinfo{number}{12} (\bibinfo{year}{2019}), \bibinfo{pages}{1181}.
\newblock


\bibitem[Krizhevsky et~al\mbox{.}(2009)]%
        {krizhevsky2009learning}
\bibfield{author}{\bibinfo{person}{Alex Krizhevsky}, \bibinfo{person}{Geoffrey
  Hinton}, {et~al\mbox{.}}} \bibinfo{year}{2009}\natexlab{}.
\newblock \showarticletitle{Learning multiple layers of features from tiny
  images}.
\newblock  (\bibinfo{year}{2009}).
\newblock


\bibitem[Letizia et~al\mbox{.}(2024)]%
        {letizia2024mutual}
\bibfield{author}{\bibinfo{person}{Nunzio~Alexandro Letizia},
  \bibinfo{person}{Nicola Novello}, {and} \bibinfo{person}{Andrea~M Tonello}.}
  \bibinfo{year}{2024}\natexlab{}.
\newblock \showarticletitle{Mutual Information Estimation via $ f $-Divergence
  and Data Derangements}.
\newblock \bibinfo{journal}{\emph{Advances in Neural Information Processing
  Systems}}  \bibinfo{volume}{37} (\bibinfo{year}{2024}),
  \bibinfo{pages}{105114--105150}.
\newblock


\bibitem[Li et~al\mbox{.}(2022)]%
        {li2022invariant}
\bibfield{author}{\bibinfo{person}{Bo Li}, \bibinfo{person}{Yifei Shen},
  \bibinfo{person}{Yezhen Wang}, \bibinfo{person}{Wenzhen Zhu},
  \bibinfo{person}{Dongsheng Li}, \bibinfo{person}{Kurt Keutzer}, {and}
  \bibinfo{person}{Han Zhao}.} \bibinfo{year}{2022}\natexlab{}.
\newblock \showarticletitle{Invariant information bottleneck for domain
  generalization}. In \bibinfo{booktitle}{\emph{Proceedings of the AAAI
  Conference on Artificial Intelligence}}, Vol.~\bibinfo{volume}{36}.
  \bibinfo{pages}{7399--7407}.
\newblock


\bibitem[Li et~al\mbox{.}(2025)]%
        {li2025contrastive}
\bibfield{author}{\bibinfo{person}{Jin Li}, \bibinfo{person}{Yaoming Wang},
  \bibinfo{person}{Xiaopeng Zhang}, \bibinfo{person}{Dongsheng Jiang},
  \bibinfo{person}{Wenrui Dai}, \bibinfo{person}{Chenglin Li},
  \bibinfo{person}{Hongkai Xiong}, {and} \bibinfo{person}{Qi Tian}.}
  \bibinfo{year}{2025}\natexlab{}.
\newblock \showarticletitle{Contrastive Learning via Variational Information
  Bottleneck}.
\newblock \bibinfo{journal}{\emph{IEEE Transactions on Pattern Analysis and
  Machine Intelligence}} (\bibinfo{year}{2025}).
\newblock


\bibitem[Liu et~al\mbox{.}(2018)]%
        {liu2018large}
\bibfield{author}{\bibinfo{person}{Ziwei Liu}, \bibinfo{person}{Ping Luo},
  \bibinfo{person}{Xiaogang Wang}, {and} \bibinfo{person}{Xiaoou Tang}.}
  \bibinfo{year}{2018}\natexlab{}.
\newblock \showarticletitle{Large-scale celebfaces attributes (celeba)
  dataset}.
\newblock \bibinfo{journal}{\emph{Retrieved August}} \bibinfo{volume}{15},
  \bibinfo{number}{2018} (\bibinfo{year}{2018}), \bibinfo{pages}{11}.
\newblock


\bibitem[Martens et~al\mbox{.}(2018)]%
        {martens2018kronecker}
\bibfield{author}{\bibinfo{person}{James Martens}, \bibinfo{person}{Jimmy Ba},
  {and} \bibinfo{person}{Matt Johnson}.} \bibinfo{year}{2018}\natexlab{}.
\newblock \showarticletitle{Kronecker-factored curvature approximations for
  recurrent neural networks}. In \bibinfo{booktitle}{\emph{International
  Conference on Learning Representations}}.
\newblock


\bibitem[Martens and Grosse(2015)]%
        {martens2015optimizing}
\bibfield{author}{\bibinfo{person}{James Martens} {and} \bibinfo{person}{Roger
  Grosse}.} \bibinfo{year}{2015}\natexlab{}.
\newblock \showarticletitle{Optimizing neural networks with kronecker-factored
  approximate curvature}. In \bibinfo{booktitle}{\emph{International conference
  on machine learning}}. PMLR, \bibinfo{pages}{2408--2417}.
\newblock


\bibitem[McAllester and Stratos(2020)]%
        {mcallester2020formal}
\bibfield{author}{\bibinfo{person}{David McAllester} {and}
  \bibinfo{person}{Karl Stratos}.} \bibinfo{year}{2020}\natexlab{}.
\newblock \showarticletitle{Formal limitations on the measurement of mutual
  information}. In \bibinfo{booktitle}{\emph{International Conference on
  Artificial Intelligence and Statistics}}. PMLR, \bibinfo{pages}{875--884}.
\newblock


\bibitem[Ollivier et~al\mbox{.}(2017)]%
        {ollivier2017information}
\bibfield{author}{\bibinfo{person}{Yann Ollivier}, \bibinfo{person}{Ludovic
  Arnold}, \bibinfo{person}{Anne Auger}, {and} \bibinfo{person}{Nikolaus
  Hansen}.} \bibinfo{year}{2017}\natexlab{}.
\newblock \showarticletitle{Information-geometric optimization algorithms: A
  unifying picture via invariance principles}.
\newblock \bibinfo{journal}{\emph{Journal of Machine Learning Research}}
  \bibinfo{volume}{18}, \bibinfo{number}{18} (\bibinfo{year}{2017}),
  \bibinfo{pages}{1--65}.
\newblock


\bibitem[Razeghi et~al\mbox{.}(2023)]%
        {razeghi2023bottlenecks}
\bibfield{author}{\bibinfo{person}{Behrooz Razeghi}, \bibinfo{person}{Flavio~P
  Calmon}, \bibinfo{person}{Deniz Gunduz}, {and} \bibinfo{person}{Slava
  Voloshynovskiy}.} \bibinfo{year}{2023}\natexlab{}.
\newblock \showarticletitle{Bottlenecks CLUB: Unifying information-theoretic
  trade-offs among complexity, leakage, and utility}.
\newblock \bibinfo{journal}{\emph{IEEE Transactions on Information Forensics
  and Security}}  \bibinfo{volume}{18} (\bibinfo{year}{2023}),
  \bibinfo{pages}{2060--2075}.
\newblock


\bibitem[Rosati et~al\mbox{.}(2024)]%
        {rosati2024representation}
\bibfield{author}{\bibinfo{person}{Domenic Rosati}, \bibinfo{person}{Jan
  Wehner}, \bibinfo{person}{Kai Williams}, \bibinfo{person}{Lukasz Bartoszcze},
  \bibinfo{person}{Robie Gonzales}, \bibinfo{person}{Subhabrata Majumdar},
  \bibinfo{person}{Hassan Sajjad}, \bibinfo{person}{Frank Rudzicz},
  {et~al\mbox{.}}} \bibinfo{year}{2024}\natexlab{}.
\newblock \showarticletitle{Representation noising: A defence mechanism against
  harmful finetuning}.
\newblock \bibinfo{journal}{\emph{Advances in Neural Information Processing
  Systems}}  \bibinfo{volume}{37} (\bibinfo{year}{2024}),
  \bibinfo{pages}{12636--12676}.
\newblock


\bibitem[Ross and Doshi-Velez(2018)]%
        {ross2018improving}
\bibfield{author}{\bibinfo{person}{Andrew Ross} {and} \bibinfo{person}{Finale
  Doshi-Velez}.} \bibinfo{year}{2018}\natexlab{}.
\newblock \showarticletitle{Improving the adversarial robustness and
  interpretability of deep neural networks by regularizing their input
  gradients}. In \bibinfo{booktitle}{\emph{Proceedings of the AAAI conference
  on artificial intelligence}}, Vol.~\bibinfo{volume}{32}.
\newblock


\bibitem[Shokri et~al\mbox{.}(2017)]%
        {shokri2017membership}
\bibfield{author}{\bibinfo{person}{Reza Shokri}, \bibinfo{person}{Marco
  Stronati}, \bibinfo{person}{Congzheng Song}, {and} \bibinfo{person}{Vitaly
  Shmatikov}.} \bibinfo{year}{2017}\natexlab{}.
\newblock \showarticletitle{Membership inference attacks against machine
  learning models}. In \bibinfo{booktitle}{\emph{2017 IEEE symposium on
  security and privacy (SP)}}. IEEE, \bibinfo{pages}{3--18}.
\newblock


\bibitem[Shwartz-Ziv and Tishby(2017)]%
        {shwartz2017opening}
\bibfield{author}{\bibinfo{person}{Ravid Shwartz-Ziv} {and}
  \bibinfo{person}{Naftali Tishby}.} \bibinfo{year}{2017}\natexlab{}.
\newblock \showarticletitle{Opening the black box of deep neural networks via
  information}.
\newblock \bibinfo{journal}{\emph{arXiv preprint arXiv:1703.00810}}
  (\bibinfo{year}{2017}).
\newblock


\bibitem[Tishby et~al\mbox{.}(2000)]%
        {tishby2000information}
\bibfield{author}{\bibinfo{person}{Naftali Tishby}, \bibinfo{person}{Fernando~C
  Pereira}, {and} \bibinfo{person}{William Bialek}.}
  \bibinfo{year}{2000}\natexlab{}.
\newblock \showarticletitle{The information bottleneck method}.
\newblock \bibinfo{journal}{\emph{arXiv preprint physics/0004057}}
  (\bibinfo{year}{2000}).
\newblock


\bibitem[Wan et~al\mbox{.}(2021)]%
        {wan2021multi}
\bibfield{author}{\bibinfo{person}{Zhibin Wan}, \bibinfo{person}{Changqing
  Zhang}, \bibinfo{person}{Pengfei Zhu}, {and} \bibinfo{person}{Qinghua Hu}.}
  \bibinfo{year}{2021}\natexlab{}.
\newblock \showarticletitle{Multi-view information-bottleneck representation
  learning}. In \bibinfo{booktitle}{\emph{Proceedings of the AAAI conference on
  artificial intelligence}}, Vol.~\bibinfo{volume}{35}.
  \bibinfo{pages}{10085--10092}.
\newblock


\bibitem[Wang et~al\mbox{.}(2024)]%
        {wang2024scu}
\bibfield{author}{\bibinfo{person}{Weiqi Wang}, \bibinfo{person}{Zhiyi Tian},
  \bibinfo{person}{Chenhan Zhang}, {and} \bibinfo{person}{Shui Yu}.}
  \bibinfo{year}{2024}\natexlab{}.
\newblock \showarticletitle{Scu: An efficient machine unlearning scheme for
  deep learning enabled semantic communications}.
\newblock \bibinfo{journal}{\emph{IEEE Transactions on Information Forensics
  and Security}} (\bibinfo{year}{2024}).
\newblock


\bibitem[Wu et~al\mbox{.}(2020)]%
        {wu2020learnability}
\bibfield{author}{\bibinfo{person}{Tailin Wu}, \bibinfo{person}{Ian Fischer},
  \bibinfo{person}{Isaac~L Chuang}, {and} \bibinfo{person}{Max Tegmark}.}
  \bibinfo{year}{2020}\natexlab{}.
\newblock \showarticletitle{Learnability for the information bottleneck}. In
  \bibinfo{booktitle}{\emph{Uncertainty in Artificial Intelligence}}. PMLR,
  \bibinfo{pages}{1050--1060}.
\newblock


\bibitem[Xie et~al\mbox{.}(2023)]%
        {xie2023robust}
\bibfield{author}{\bibinfo{person}{Songjie Xie}, \bibinfo{person}{Shuai Ma},
  \bibinfo{person}{Ming Ding}, \bibinfo{person}{Yuanming Shi},
  \bibinfo{person}{Mingjian Tang}, {and} \bibinfo{person}{Youlong Wu}.}
  \bibinfo{year}{2023}\natexlab{}.
\newblock \showarticletitle{Robust information bottleneck for task-oriented
  communication with digital modulation}.
\newblock \bibinfo{journal}{\emph{IEEE Journal on Selected Areas in
  Communications}} \bibinfo{volume}{41}, \bibinfo{number}{8}
  (\bibinfo{year}{2023}), \bibinfo{pages}{2577--2591}.
\newblock


\bibitem[Xu et~al\mbox{.}(2022)]%
        {xu2022infoat}
\bibfield{author}{\bibinfo{person}{Mengting Xu}, \bibinfo{person}{Tao Zhang},
  \bibinfo{person}{Zhongnian Li}, {and} \bibinfo{person}{Daoqiang Zhang}.}
  \bibinfo{year}{2022}\natexlab{}.
\newblock \showarticletitle{InfoAT: Improving adversarial training using the
  information bottleneck principle}.
\newblock \bibinfo{journal}{\emph{IEEE Transactions on Neural Networks and
  Learning Systems}} \bibinfo{volume}{35}, \bibinfo{number}{1}
  (\bibinfo{year}{2022}), \bibinfo{pages}{1255--1264}.
\newblock


\bibitem[Xu et~al\mbox{.}(2024)]%
        {xu2024sctnet}
\bibfield{author}{\bibinfo{person}{Zhengze Xu}, \bibinfo{person}{Dongyue Wu},
  \bibinfo{person}{Changqian Yu}, \bibinfo{person}{Xiangxiang Chu},
  \bibinfo{person}{Nong Sang}, {and} \bibinfo{person}{Changxin Gao}.}
  \bibinfo{year}{2024}\natexlab{}.
\newblock \showarticletitle{Sctnet: Single-branch cnn with transformer semantic
  information for real-time segmentation}. In
  \bibinfo{booktitle}{\emph{Proceedings of the AAAI conference on artificial
  intelligence}}, Vol.~\bibinfo{volume}{38}. \bibinfo{pages}{6378--6386}.
\newblock


\bibitem[Yang et~al\mbox{.}(2025)]%
        {yang2025structured}
\bibfield{author}{\bibinfo{person}{Hanzhe Yang}, \bibinfo{person}{Youlong Wu},
  \bibinfo{person}{Dingzhu Wen}, \bibinfo{person}{Yong Zhou}, {and}
  \bibinfo{person}{Yuanming Shi}.} \bibinfo{year}{2025}\natexlab{}.
\newblock \showarticletitle{Structured IB: Improving Information Bottleneck
  with Structured Feature Learning}. In \bibinfo{booktitle}{\emph{Proceedings
  of the AAAI Conference on Artificial Intelligence}},
  Vol.~\bibinfo{volume}{39}. \bibinfo{pages}{21922--21928}.
\newblock


\bibitem[Yu et~al\mbox{.}(2024)]%
        {yu2024cauchyschwarz}
\bibfield{author}{\bibinfo{person}{Shujian Yu}, \bibinfo{person}{Xi Yu},
  \bibinfo{person}{Sigurd L{\o}kse}, \bibinfo{person}{Robert Jenssen}, {and}
  \bibinfo{person}{Jose~C Principe}.} \bibinfo{year}{2024}\natexlab{}.
\newblock \showarticletitle{Cauchy-Schwarz Divergence Information Bottleneck
  for Regression}. In \bibinfo{booktitle}{\emph{The Twelfth International
  Conference on Learning Representations}}.
\newblock


\bibitem[Zhai and Zhang(2022)]%
        {zhai2022adversarial}
\bibfield{author}{\bibinfo{person}{Penglong Zhai} {and} \bibinfo{person}{Shihua
  Zhang}.} \bibinfo{year}{2022}\natexlab{}.
\newblock \showarticletitle{Adversarial information bottleneck}.
\newblock \bibinfo{journal}{\emph{IEEE Transactions on Neural Networks and
  Learning Systems}} \bibinfo{volume}{35}, \bibinfo{number}{1}
  (\bibinfo{year}{2022}), \bibinfo{pages}{221--230}.
\newblock


\end{thebibliography}

\clearpage
\appendix

\section{Proof of Equation \ref{eq:ig-pyth}}  \label{proof_of_ig_pyth}

\begin{proof}
 	Using the log-factorization identity and Fubini/Tonelli
	\begin{align*}
		\mathrm{KL} & \!\big(p(x,z)\, \|\,q(x)r(z)\big)
		= \iint p(x,z)\,\log\frac{p(x,z)}{q(x)r(z)}\,dx\,dz \\
		&= \iint p(x,z)\,\log\frac{p(x,z)}{p_X(x)p_Z(z)}\,dx\,dz \\
		&\quad + \iint p(x,z)\,\log\frac{p_X(x)}{q(x)}\,dx\,dz
		+ \iint p(x,z)\,\log\frac{p_Z(z)}{r(z)}\,dx\,dz \\
		&= \underbrace{\mathrm{KL}\!\big(p(x,z)\,\|\,p_X(x)p_Z(z)\big)}_{=\,I(X;Z)} \\
		&\quad + \underbrace{\int p_X(x)\,\log\frac{p_X(x)}{q(x)}\,dx}_{=\,\mathrm{KL}(p_X\|q)} 
		+ \underbrace{\int p_Z(z)\,\log\frac{p_Z(z)}{r(z)}\,dz}_{=\,\mathrm{KL}(p_Z\|r)}.
	\end{align*}
	The last step uses Fubini's theorem to integrate out $z$ and $x$ respectively.
	Because each KL term is nonnegative, the minimum over $q,r$ is achieved at $q=p_X,\ r=p_Z$, with value $I(X;Z)$.
\end{proof}

\begin{remark}[Information-geometric Pythagorean relation]
	Since $\mathcal{I}$ is flat in the appropriate dual affine coordinates, the above decomposition
	is also the IG ``Pythagorean theorem'':
	\[
	\mathrm{KL}\!\big(p\,\|\,q r\big)
	= \mathrm{KL}\!\big(p\,\|\,p_X p_Z\big)
	+ \mathrm{KL}\!\big(p_X p_Z\,\|\,q r\big), \qquad q r \in \mathcal{I}.
	\]
	Thus $p_X p_Z$ is the e-projection of $p$ onto $\mathcal{I}$, and $I(X;Z)$ is exactly the
	projection distance.
\end{remark}

\section{Proof of local second–order (FR).}
\label{proof_of_FR}

We show that for a regular parametric family $\{p_\theta : \theta\in\Theta\subset\mathbb{R}^d\}$ and $\theta'$ near $\theta$, i.e., $\Delta = \theta' - \theta$,
\[
\mathrm{KL}\!\big(p_{\theta}\,\|\,p_{\theta'}\big)
= \tfrac12\,\Delta^\top F(\theta)\,\Delta \;+\; o(\|\Delta\|^2)
= \tfrac12\, d_{\text{FR}}\!\big(p_{\theta},p_{\theta'}\big)^2 \;+\; o(\|\Delta\|^2),
\]
where $F(\theta)$ is the Fisher information and $d_{\text{FR}}$ is the Fisher–Rao distance. Throughout we assume standard regularity: common support, differentiability up to second order in $\theta$, finiteness of $F(\theta)$, and interchange of expectation and differentiation.

Write
\[
\mathrm{KL}\!\big(p_{\theta}\,\|\,p_{\theta'}\big)
= \E_{p_\theta}\!\left[\log p_\theta(Z)-\log p_{\theta'}(Z)\right].
\]
Fix $z$ and expand $\log p_{\theta'}(z)$ at $\theta$:
\[
\log p_{\theta'}(z)
= \log p_\theta(z) + \Delta^\top \nabla_\theta \log p_\theta(z)
+ \tfrac12\,\Delta^\top \nabla^2_\theta \log p_\theta(z)\,\Delta
+ o(\|\Delta\|^2).
\]
Subtract from $\log p_\theta(z)$, take $\E_{p_\theta}$, and use
$\E_{p_\theta}\!\big[\nabla_\theta \log p_\theta(Z)\big]=0$ (zero mean score) to obtain
\begin{align*}
	\mathrm{KL}\!\big(p_{\theta}\,\|\,p_{\theta'}\big)
	&= -\tfrac12\,\Delta^\top \E_{p_\theta}\!\big[\nabla^2_\theta \log p_\theta(Z)\big]\,\Delta
	\;+\; o(\|\Delta\|^2).
\end{align*}
By the information identity,
\[
F(\theta) \;=\; \E_{p_\theta}\!\big[ \nabla_\theta \log p_\theta(Z)\,\nabla_\theta \log p_\theta(Z)^\top \big]
\;=\; -\,\E_{p_\theta}\!\big[\nabla^2_\theta \log p_\theta(Z)\big],
\]
hence
\begin{equation}
	\label{eq:kl-second-order}
	\mathrm{KL}\!\big(p_{\theta}\,\|\,p_{\theta'}\big)
	= \tfrac12\,\Delta^\top F(\theta)\,\Delta \;+\; o(\|\Delta\|^2).
\end{equation}


The FR metric is the Riemannian metric on $\Theta$ given by
$g_\theta(u,v) = u^\top F(\theta)\,v$ for tangent vectors $u,v\in\mathbb{R}^d$.
Let $\gamma:[0,1]\to\Theta$ be any $C^1$ curve with $\gamma(0)=\theta$ and $\gamma(1)=\theta'$.
Its FR length is
\[
L(\gamma) \;=\; \int_0^1 \sqrt{\dot\gamma(t)^\top F(\gamma(t))\,\dot\gamma(t)}\,dt,
\]
and the FR distance is $d_{\text{FR}}(p_\theta,p_{\theta'})=\inf_\gamma L(\gamma)$.
For $\|\Delta\|\to 0$, choose the straight segment $\gamma(t)=\theta+t\Delta$ to get
\[
L(\gamma)^2
= \Big(\int_0^1 \sqrt{\Delta^\top F(\theta+t\Delta)\,\Delta}\,dt\Big)^2
= \Delta^\top F(\theta)\,\Delta \;+\; o(\|\Delta\|^2),
\]
using continuity of $F$ and a second–order expansion in $t$. Since the geodesic length is minimal,
\begin{equation}
	\label{eq:fr-local}
	d_{\text{FR}}\!\big(p_\theta,p_{\theta'}\big)^2
	= \Delta^\top F(\theta)\,\Delta \;+\; o(\|\Delta\|^2).
\end{equation}
This follows from standard Riemannian geometry: in normal coordinates at $\theta$, $d_{\text{FR}}(p_\theta,p_{\theta'})^2=\Delta^\top F(\theta)\Delta+O(\|\Delta\|^3)$, hence \eqref{eq:fr-local}. 
Combining \eqref{eq:kl-second-order} and \eqref{eq:fr-local} yields the local equivalence
\[
\mathrm{KL}\!\big(p_{\theta}\,\|\,p_{\theta'}\big)
= \tfrac12\, d_{FR}\!\big(p_{\theta},p_{\theta'}\big)^2 \;+\; o(\|\Delta\|^2).
\]


Let $q_\phi(z|x)$ and a reference marginal $r(z)$ belong to $\{p_\theta\}$ with parameters
$\theta(x)$ and $\theta_r$ respectively, and assume $\|\theta(x)-\theta_r\|$ is small for $p(x)$–almost every $x$. Applying the pointwise result above with $\Delta(x)=\theta(x)-\theta_r$ gives
\[
\mathrm{KL}\!\big(q_\phi(z|x)\,\|\,r(z)\big)
= \tfrac12\, d_{\text{FR}}\!\big(q_\phi(z|x), r(z)\big)^2 \;+\; o\!\big(\|\Delta(x)\|^2\big).
\]
Taking $\E_{p(x)}$ and using dominated convergence (guaranteed by the regularity assumptions) yields
\begin{align*}
\E_{p(x)}\mathrm{KL}\!\big(q_\phi(z|x)\,\|\,r(z)\big)
= & \tfrac12\, \E_{p(x)} d_{\text{FR}}\!\big(q_\phi(z|x), r(z)\big)^2 \; \\
 &+\; o\!\Big(\E_{p(x)}\|\Delta(x)\|^2\Big).
\end{align*}
When $r(z)=p_\phi(z)$ (the aggregate posterior), this is precisely the local second–order approximation of the compression term
$\E_{p(x)}\mathrm{KL}\!\big(q_\phi(z|x)\,\|\,p_\phi(z)\big)$. \qed

\section{Proof of Theorem 1}\label{proof_of_t1}

\begin{proof}[Proof of Theorem 1]  
	Let $v:=-\eta\,\mathrm{grad}\,\mathcal J(\phi)\in T_\phi\mathcal M$.
	By existence and uniqueness for the geodesic equation with the Levi--Civita connection of the FR metric,
	there exists $\varepsilon>0$ and a unique geodesic $\gamma_v:(-\varepsilon,\varepsilon)\to\mathcal M$
	such that $\gamma_v(0)=\phi$ and $\dot\gamma_v(0)=v$.
	
	The Riemannian exponential map at $\phi$ is defined by
	\[
	\Exp_\phi(w)\;=\;\gamma_w(1)\quad\text{whenever $1$ lies in the domain of }\gamma_w,
	\]
	equivalently $\Exp_\phi(tw)=\gamma_w(t)$ for $t$ in a neighborhood of $0$.
	(For global well-definedness one may restrict to $\|w\|$ below the injectivity radius.)
	Applying this with $w=v$ gives
	\[
	\phi^{+}\;=\;\Exp_\phi(v)\;=\;\gamma_v(1),
	\]
	i.e., $\phi^{+}$ lies on the unique FR geodesic starting at $\phi$ with initial velocity
	$\dot\gamma(0)=v=-\eta\,\mathrm{grad}\,\mathcal J(\phi)$.
\end{proof}

\section{The GeoIB Algorithm} \label{alg_of_gib}

We provide the algorithm of GeoIB as the following \Cref{alg:gib-step}.

 \begin{algorithm}[h]
 	\caption{GeoIB Natural-Gradient Step (per iteration)}\label{alg:gib-step}
 	\begin{small}
 		\BlankLine
 		\KwIn{Current params $(\phi_t,\theta_t)$; minibatch size $B$; step sizes $(\eta_\phi,\eta_\theta)$; bottleneck $\beta$; \# Hutchinson probes $S$; damping $\lambda$; Fisher approx. mode \textsf{K-FAC}}
 		\KwOut{Updated params $(\phi_{t+1},\theta_{t+1})$}
 		\SetKwFunction{GIBSTEP}{\textbf{GIB\_Step}}
 		\SetKwProg{Fn}{procedure}{:}{end procedure}
 		\SetNlSty{}{}{}
 		\Fn{\GIBSTEP{$\phi_t,\theta_t,B,\eta_\phi,\eta_\theta,\beta,S,\lambda,\textsf{mode}$}}{
 			\tcp{1) Sample minibatch and latent codes}
 			Draw minibatch $\{(x_i,y_i)\}_{i=1}^B$;\; 
 			$z_i \sim q_{\phi_t}(z\mid x_i)$ by reparameterization\;
 			\BlankLine
 			\tcp{2) Estimate FR/JF terms (Hutchinson + JVPs)}
 			Estimate $\widehat{\mathcal L}_{\mathrm{FR}}$ and $\widehat{\mathcal L}_{\mathrm{JF}}$ using $S$ probe vectors and JVPs\;
 			\BlankLine
 			\tcp{3) Compute Euclidean gradients}
 			$g_\theta \gets \nabla_{\theta}\,\frac{1}{B}\sum_{i=1}^B \big[-\log p_{\theta_t}(y_i\mid z_i)\big]$\;
 			$g_\phi \gets \nabla_{\phi}\,\Big\{\frac{1}{B}\sum_{i=1}^B \big[-\log p_{\theta_t}(y_i\mid z_i)\big] 
 			+ \beta\big(\widehat{\mathcal L}_{\mathrm{FR}}+\widehat{\mathcal L}_{\mathrm{JF}}\big)\Big\}$\;
 			\BlankLine
 			\tcp{4) Build Fisher approximations (Empirical Fisher or K-FAC)}
 			\uIf{\textsf{mode} = \textsf{K-FAC}}{
 				Build layerwise Kronecker factors for encoder/decoder to obtain $\widehat F_\phi$ and $\widehat F_\theta$\;
 			}
 			$\widehat F_\theta^\lambda \gets \widehat F_\theta + \lambda I$, \quad $\widehat F_\phi^\lambda \gets \widehat F_\phi + \lambda I$\;
 			\BlankLine
 			\tcp{5) Solve for natural directions (no explicit inversion)}
 			Find $v_\theta$ s.t. $(\widehat F_\theta^\lambda)\,v_\theta = g_\theta$ via Conjugate Gradient (CG)\;
 			Find $v_\phi$ s.t. $(\widehat F_\phi^\lambda)\,v_\phi = g_\phi$ via CG\;
 			\BlankLine
 			\tcp{6) Parameter updates (natural gradients)}
 			$\theta_{t+1} \gets \theta_t - \eta_\theta\, v_\theta$ \hfill (\Cref{eq:decoder-natgrad})\;
 			$\phi_{t+1} \gets \phi_t - \eta_\phi\, v_\phi$ \hfill (\Cref{eq:natgrad-step})\;
 			\BlankLine
 			\Return $(\phi_{t+1},\theta_{t+1})$\;
 		}
 	\end{small}
 \end{algorithm}
 
Algorithm~\ref{alg:gib-step} executes one GeoIB training iteration with \emph{natural-gradient} updates for the decoder $\theta$ and encoder $\phi$ under a K-FAC curvature mode. 
(Step 1) Given a minibatch $\{(x_i,y_i)\}_{i=1}^B$, latent codes are drawn via the reparameterized posterior $z_i\!\sim q_{\phi_t}(z\mid x_i)$. 
(Step 2) The geometry-aware bottleneck surrogates are estimated without forming full Jacobians: the Fisher–Rao proxy $\widehat{\mathcal L}_{\mathrm{FR}}$ and the Jacobian–Frobenius penalty $\widehat{\mathcal L}_{\mathrm{JF}}$ are computed using $S$ Hutchinson probe vectors together with Jacobian–vector products (JVPs). 
(Step 3) We then compute the Euclidean gradients of the decoder NLL and the full encoder objective,
\begin{align*}
	g_\theta=\nabla_\theta\,\tfrac{1}{B}\!\sum_{i=1}^B\!\big[-\log p_{\theta_t}(y_i\mid z_i)\big],\qquad \\
	g_\phi=\nabla_\phi\!\left\{\tfrac{1}{B}\!\sum_{i=1}^B\!\big[-\log p_{\theta_t}(y_i\mid z_i)\big]+\beta\!\big(\widehat{\mathcal L}_{\mathrm{FR}}+\widehat{\mathcal L}_{\mathrm{JF}}\big)\right\}.
\end{align*}

(Step 4) To obtain natural directions, we build \emph{layerwise} Kronecker-factored Fisher approximations for both networks. For each layer $\ell$ with weight matrix $W_\ell$, K-FAC uses the block-diagonal model
$F_\ell \approx A_\ell \otimes G_\ell$,
where $A_\ell:=\tfrac{1}{B}\sum_i a_{\ell,i}a_{\ell,i}^\top$ is the covariance of layer inputs (activations) and $G_\ell:=\tfrac{1}{B}\sum_i g_{\ell,i}g_{\ell,i}^\top$ is the covariance of backpropagated pre-activation gradients. Tikhonov damping yields $\widehat F_\theta^\lambda=\widehat F_\theta+\lambda I$ and $\widehat F_\phi^\lambda=\widehat F_\phi+\lambda I$. 
(Step 5) Rather than inverting these matrices, we solve the linear systems
$(\widehat F_\theta^\lambda)v_\theta=g_\theta$ and $(\widehat F_\phi^\lambda)v_\phi=g_\phi$ via conjugate gradients (CG). Each CG iteration only needs Fisher–vector products, which K-FAC supplies efficiently: if $V_\ell$ reshapes the vector $v$ to the layer’s weight shape, then
\[
\mathrm{FVP}_\ell(v)\;=\;\mathrm{vec}\!\Big((G_\ell+\lambda I)\,V_\ell\,(A_\ell+\lambda I)\Big).
\]
(Step 6) Finally, parameters are updated along the natural directions:
$\theta_{t+1}=\theta_t-\eta_\theta v_\theta$ (cf.~\Cref{eq:decoder-natgrad}) and
$\phi_{t+1}=\phi_t-\eta_\phi v_\phi$ (cf.~\Cref{eq:natgrad-step}). 
The hyperparameters $(\eta_\phi,\eta_\theta,\beta,S,\lambda)$ control step sizes, compression strength, estimator variance, and conditioning, while K-FAC trades a faithful curvature signal for scalable, inversion-free natural-gradient steps.

\begin{table}[h]
	\scriptsize
	\caption{Dataset statistics.} 
	\label{dataset_table}
	\resizebox{\linewidth}{!}{
		\setlength\tabcolsep{7.5pt}
		\begin{tabular}{cccc}
			\toprule[0.8pt]
			Dataset & Feature Dimension  & \#. Classes & \#. Samples \\
			\midrule
			MNIST \citep{deng2012mnist} & 28×28×1 & 10 & 70,000  \\  
			\rowcolor{verylightgray}
			CIFAR10 \citep{krizhevsky2009learning} & 32×32×3 & 10 & 60,000  \\  
			CelebA~\citep{liu2018large} & 178×218×3 & 2 (Gender) & 202,599 \\
			\bottomrule[0.8pt]
	\end{tabular}}
	\vspace{-2mm}
\end{table}


\section{Datasets} \label{datasets_appendix}

The statistics of all datasets used in our experiments are listed in \Cref{dataset_table}. Both MNIST and CIFAR10 are used to train 10-class classification models. The experiment on CelebA is to identify the gender attributes of the face images. The task is a binary classification problem, different from the ones on MNIST and CIFAR10. These datasets offer a range of objective categories with varying levels of learning complexity. We also introduce them as below.



\begin{itemize}[itemsep=0pt, parsep=0pt, leftmargin=*]
	\item \textbf{MNIST~\citep{deng2012mnist}.} MNIST contains 60,000 handwritten digit images for the training and 10,000 handwritten digit images for the testing. All these black and white digits are size normalized, and centered in a fixed-size image with 28 × 28 pixels.
	\item \textbf{CIFAR10~\citep{krizhevsky2009learning}.} CIFAR10 dataset consists of 60,000 32x32 colour images in 10 classes, with 6,000 images per class. There are 50,000 training images and 10,000 test images.
	\item  \textbf{CelebA~\citep{liu2018large}.} CelebA is a large-scale face attributes dataset with more than 200,000 celebrity images, each with 40 attribute annotations.
\end{itemize}

\begin{figure*}[t] 
	\centering
	\vspace{-4mm}
	\subfloat[\scriptsize $\beta=10^{-4}$, Accuracy $=86.05\%$]{ 	\label{fig:repdemocifar00001}
		\includegraphics[scale=0.34]{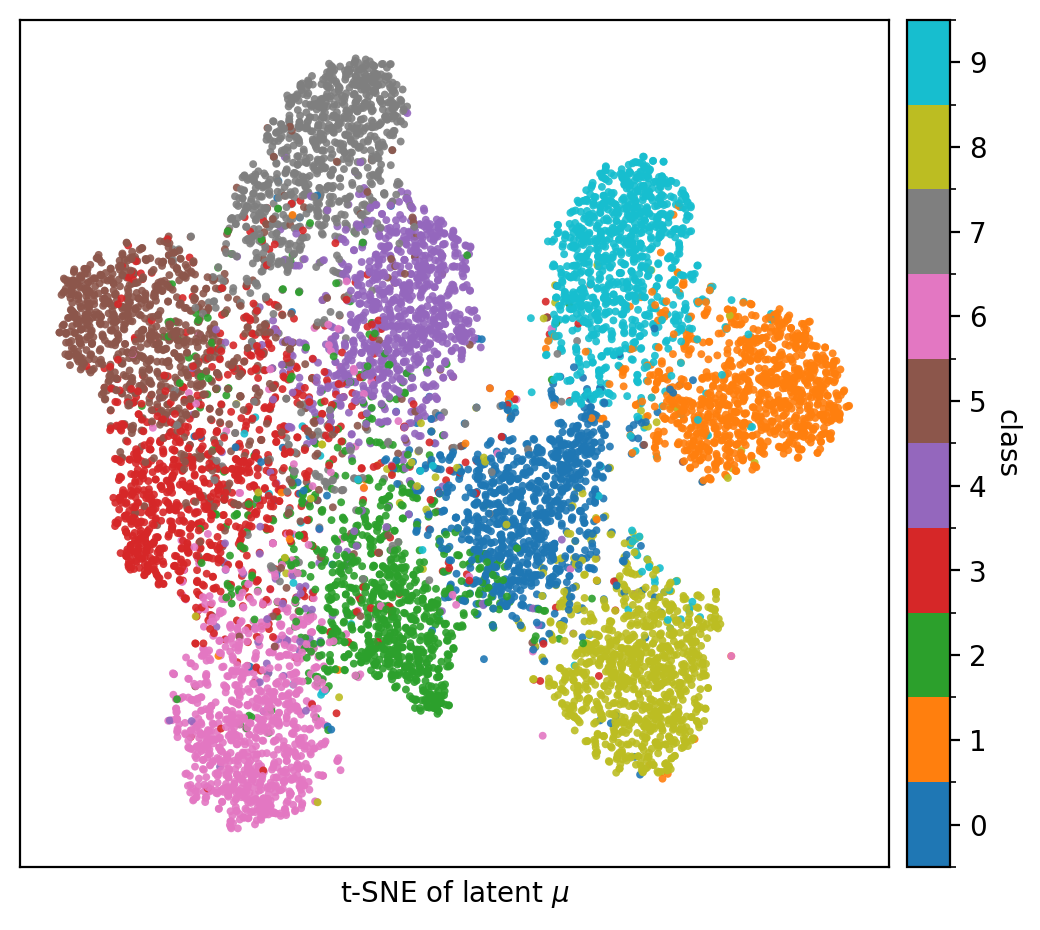}
	}			
	\subfloat[\scriptsize $\beta=10^{0}$, Accuracy $=82.05\%$]{ 	 	\label{fig:repdemocifar2}
		\includegraphics[scale=0.34]{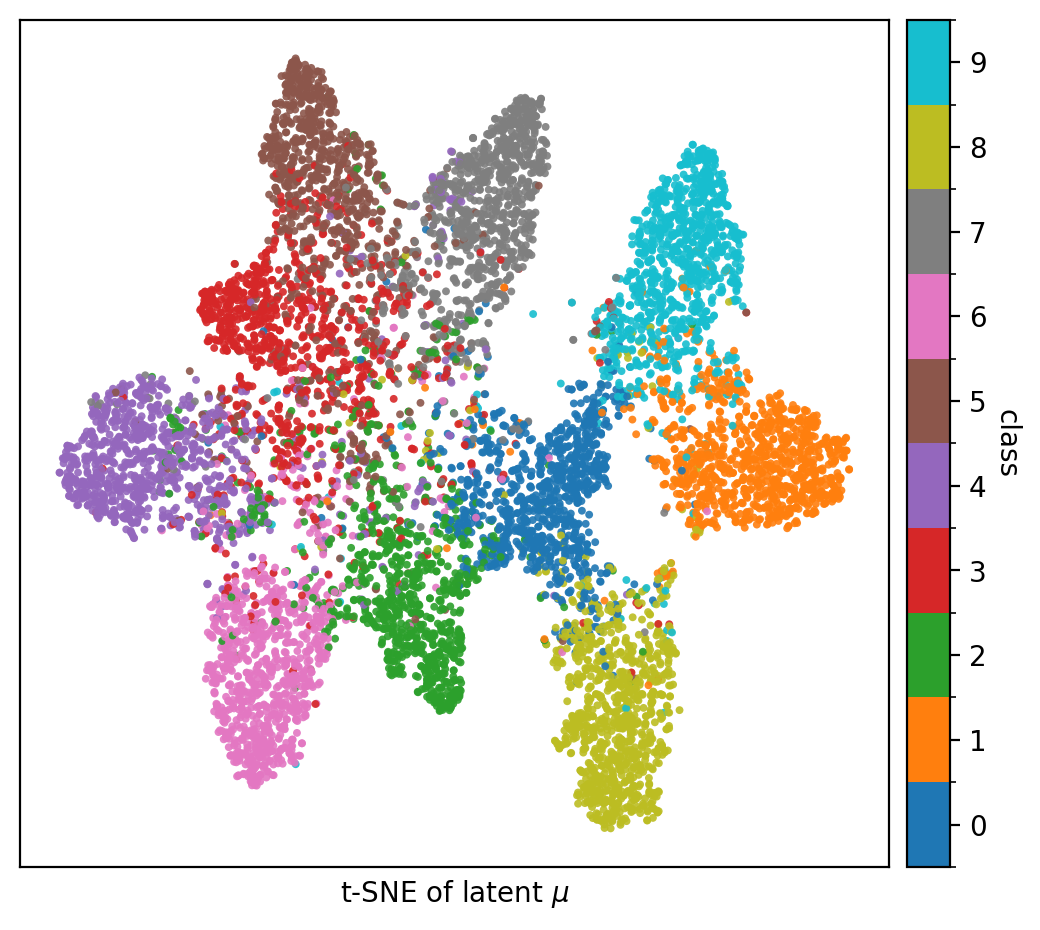}
	}
	\subfloat[\scriptsize $\beta=10^{1}$, Accuracy $=76.11\%$]{   	\label{fig:repdemocifar10}
		\includegraphics[scale=0.34]{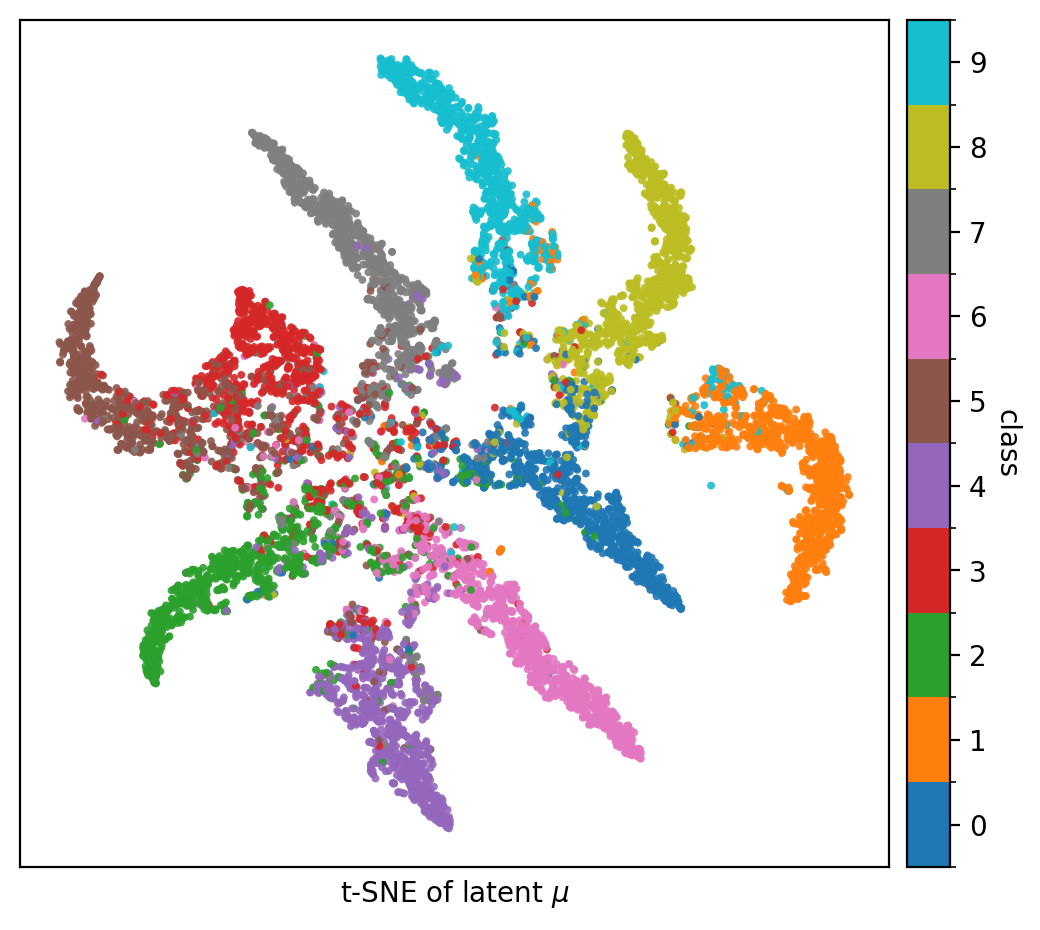}
	}  
	\caption{Visualizing representation embeddings of 10000 test images in two dimensions on CIFAR10. The images are colored according to their true class label. We $\beta$ becomes larger, we forget more about the input and the representation embedding of each class is compressed close to the average $\mu$. We also report the test accuracy, which decreases as $\beta$ increases.} 
	\label{compressive_on_cifar10} 
\end{figure*}

\section{Additional Experiments} \label{addtional_exp}

\noindent\textbf{Qualitative visualization.}
To examine how the bottleneck strength shapes the representation, we visualize a 2-D t-SNE of the representation $\mu(x)=\mathbb{E}[ Z \mid x]$ for 10{,}000 CIFAR10 test images at three values of $\beta$ (\Cref{compressive_on_cifar10}). With a small Bottleneck multiplier ($\beta=10^{-4}$), clusters are well separated and exhibit relatively large within-class spread, indicating that $Z$ retains fine-grained input details. At $\beta=10^{0}$, representation clusters contract toward class-wise prototypes while remaining separable. At $\beta=10^{1}$, representation embeddings concentrate along narrow arcs near the class centers, consistent with stronger compression and the accuracy drop observed in \Cref{evaluation_of_general_beta}.

\end{document}